\documentclass{article} % For LaTeX2e
\usepackage{iclr2026_conference,times}
\usepackage{amsmath,amsfonts,bm}

\def\eqref#1{equation~\ref{#1}}
\def\1{\bm{1}}

\DeclareMathAlphabet{\mathsfit}{\encodingdefault}{\sfdefault}{m}{sl}
\SetMathAlphabet{\mathsfit}{bold}{\encodingdefault}{\sfdefault}{bx}{n}

\usepackage{titletoc}
\usepackage{enumitem}

\usepackage{booktabs}  % For \toprule, \midrule, \bottomrule
\usepackage{multirow}  % For \multirow command
\usepackage{multicol}
\usepackage{makecell}
\usepackage{array}     % For enhanced array and tabular environments
\usepackage{graphicx}  % For \resizebox command
\usepackage{tabularx}
\usepackage{wrapfig}

\usepackage{hyperref}
\usepackage{url}
\usepackage{xcolor}
\usepackage{colortbl}
\newcommand{\normalrow}{\rowcolor{gray!30}}
\definecolor{deltaBg}{RGB}{220,230,255} % subtle blue-grey
\newcommand{\rowhighlight}{\rowcolor{deltaBg}}

\usepackage{tcolorbox}
\tcbuselibrary{breakable}
\definecolor{my_green}{RGB}{51,102,0}
\definecolor{my_purple}{RGB}{160, 43, 147}
\definecolor{my_blue}{RGB}{0, 0, 0}
\usepackage{dashrule} 
\usepackage{wasysym}
\usepackage{marvosym}

\usepackage{amsthm}
\usepackage{amssymb}
\newtheorem{theorem}{Theorem}

\NewDocumentCommand{\yafu}
{ mO{} }{\textcolor{red}{\textsuperscript{\textit{yafu}}\textsf{\textbf{\small[#1]}}}}

\title{Diversity-Incentivized Exploration for Versatile Reasoning}

\author{Zican Hu\textsuperscript{12$\ast$} \quad 
Shilin Zhang\textsuperscript{12$\ast$} \quad
Yafu Li\textsuperscript{24$\dagger$\textrm{\Letter}} \quad
Jianhao Yan\textsuperscript{23} \quad
Xuyang Hu\textsuperscript{2} \quad \\
\textbf{Leyang Cui}\textsuperscript{3} \quad 
\textbf{Xiaoye Qu}\textsuperscript{\textbf{2}} \quad 
\textbf{Chunlin Chen}\textsuperscript{\textbf{1}} \quad
\textbf{Yu Cheng}\textsuperscript{\textbf{4\textrm{\Letter}}} \quad
\textbf{Zhi Wang}\textsuperscript{\textbf{12\textrm{\Letter}}}
\\
\textsuperscript{1} Nanjing University
\textsuperscript{2} Shanghai AI Laboratory\\
\textsuperscript{3} Westlake University
\textsuperscript{4} The Chinese University of Hong Kong
\\
\texttt{\{zicanhu,shilinzhang\}@smail.nju.edu.cn} \quad \texttt{yafuly@gmail.com}\\
\texttt{chengyu@cse.cuhk.edu.hk} \quad \texttt{zhiwang@nju.edu.cn}
}

\iclrfinalcopy % Uncomment for camera-ready version, but NOT for submission.

\begin{document}
\maketitle
\renewcommand{\thefootnote}{}
\footnotetext{$^\ast$Equal contributions. Zican Hu and Shilin Zhang are listed alphabetically by last name. This work was conducted during internship at Shanghai AI Laboratory. $^\textrm{\Letter}$Corresponding authors. $^\dagger$Project lead.}
\renewcommand{\thefootnote}{\arabic{footnote}}

\begin{abstract}
Reinforcement Learning with Verifiable Rewards (RLVR) has emerged as a crucial paradigm for incentivizing reasoning capabilities in Large Language Models (LLMs).
Due to vast state-action spaces and reward sparsity in reasoning tasks, existing methods often struggle with deficient exploration and poor sample efficiency.
In the paper, we propose \textbf{DIVER} (\textbf{D}iversity-\textbf{I}ncentivized Exploration for \textbf{V}ersatil\textbf{E} \textbf{R}easoning), an innovative framework that highlights the pivotal role of global sequence-level diversity to incentivize deep exploration for versatile reasoning.
We first conduct a primary empirical study to reveal a strong positive correlation between global diversity and reasoning capacity.
Building on this insight, we introduce global diversity incentives as an intrinsic reward to promote deep exploration in a semantically structured space.
Incorporating the intrinsic reward, we develop a potential-based reward shaping mechanism to preserve optimal policy invariance and design simple heuristics to mitigate possible reward hacking.
Experimental results show that DIVER outperforms competitive RLVR baselines with various exploration strategies on both in-domain and out-of-domain tasks, excelling in both Pass@1 and Pass@k evaluations.
Our code is available at \textcolor{magenta}{\href{https://github.com/NJU-RL/DIVER}{https://github.com/NJU-RL/DIVER}}.

% Reinforcement Learning with Verifiable Rewards (RLVR) has emerged as a crucial paradigm for incentivizing reasoning capabilities in Large Language Models (LLMs). Due to vast state-action spaces and reward sparsity in reasoning tasks, existing methods often struggle with deficient exploration and poor sample efficiency. In the paper, we propose **DIVER** (**D**iversity-**I**ncentivized Exploration for **V**ersatil**E** **R**easoning), an innovative framework that highlights the pivotal role of global sequence-level diversity to incentivize deep exploration for versatile reasoning. We first conduct a primary empirical study to reveal a strong positive correlation between global diversity and reasoning capacity. Building on this insight, we introduce global diversity incentives as an intrinsic reward to promote deep exploration in a semantically structured space. Incorporating the intrinsic reward, we develop a potential-based reward shaping mechanism to preserve optimal policy invariance and design simple heuristics to mitigate possible reward hacking. Experimental results show that DIVER outperforms competitive RLVR baselines with various exploration strategies on both in-domain and out-of-domain tasks, excelling in both Pass@1 and Pass@k evaluations.

\end{abstract}

% Reinforcement Learning with Verifiable Rewards (RLVR) has emerged as a crucial paradigm for incentivizing reasoning capabilities in Large Language Models (LLMs). Due to vast state-action spaces and reward sparsity in reasoning tasks, existing methods often struggle with deficient exploration and poor sample efficiency. In the paper, we propose **DIVER** (**D**iversity-**I**ncentivized Exploration for **V**ersatil**E** **R**easoning), an innovative framework that highlights the pivotal role of global sequence-level diversity to incentivize deep exploration for versatile reasoning. We first conduct a primary empirical study to reveal a strong positive correlation between global diversity and reasoning capacity. Built on this insight, we propose a systematic scheme that injects diversity incentives to facilitate deep exploration in a semantically structured space from representative perspectives: sampling, prompting, and optimization. For the optimization side, we design a potential-based reward shaping mechanism to preserve optimal policy invariance when incorporating diversity as an intrinsic reward, and introduce simple heuristics to mitigate reward hacking. Experimental results show that DIVER outperforms competitive RLVR baselines with various exploration strategies on both in-domain and out-of-domain tasks, excelling in Pass@1 and Pass@k evaluations.

\section{Introduction}
% \yafu{Generally, the groundwork established in the introduction can be a bit more concise. For instance, the second and third para. can be merged, and the previous work in traditional RL and action-level diversity in LRMs can be shorter (detailed discussion can be put into Related work). Leave some space to introduce DIVER method as well as empirical findings.}

\begin{wrapfigure}{r}{0.52\textwidth}
  \centering 
  \includegraphics[width=\linewidth]{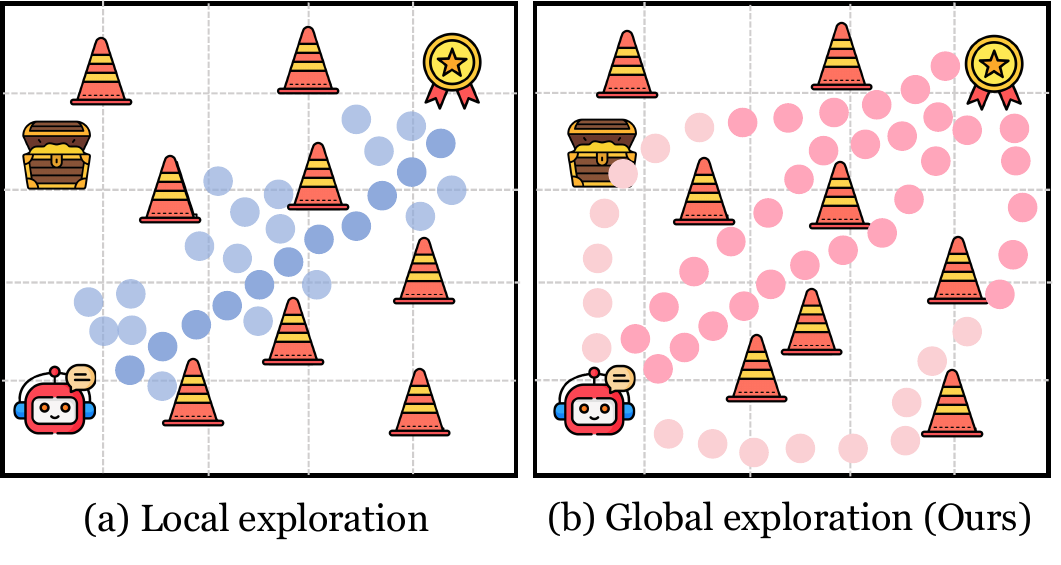}
  \caption{Local token-level vs. Global sequence-level exploration.
  We incentivize deep exploration to broaden diverse pathways for versatile reasoning.
  }
  \label{fig:idea}
\end{wrapfigure}

%%% Storyline %%%
% 1. RLVR exploration in a huge space, low efficiency
% 2. Existing methods for promoting exploration: primitive action-level diversity, like entropy
% 3. Our method: behavior-level diversity, sophisticated exploration and structured reasoning

Reinforcement Learning with Verifiable Rewards (RLVR) has advanced reasoning capabilities in Large Language Models (LLMs) through rule-based verification on model's responses~\citep{r1,orz,simplerlzoo}.
% A central challenge is the fundamental exploration–exploitation tradeoff highlighted in classic RL literature~\citep{lillicrap2015continuous,fujimoto2018addressing, haarnoja2018soft,haarnoja2017reinforcement}.
A central challenge is the fundamental exploration–exploitation tradeoff highlighted in classic RL literature~\citep{lillicrap2015continuous, haarnoja2018soft}.
Unlike traditional RL environments with relatively small, well-defined state-action spaces~\citep{sutton2018reinforcement,wang2024meta,zhang2025text}, LLM policies operate in vast, high-dimensional textual spaces with complex semantics, where the number of possible state-action pairs grows exponentially with sequence length~\citep{gupta2024language,ahn2024large,hu2025divide}. 
% This combinatorial explosion makes conventional exploration strategies computationally intractable. 
This combinatorial explosion greatly increases the difficulty of effective exploration in textual reasoning, especially under limited computing resources.
% Furthermore, the inherent reward sparsity in challenging reasoning tasks creates massive ``reward deserts" where the model receives little to no meaningful feedback most of the time~\citep{haarnoja2017reinforcement}, thereby hindering the discovery of improved pathways through the trial-and-error mechanism~\citep{}.
Furthermore, the inherent reward sparsity in challenging reasoning tasks creates massive ``reward deserts" where the model receives little to no meaningful feedback most of the time~\citep{haarnoja2017reinforcement,wu2025mixture}, hindering the discovery of improved pathways due to the lack of valid feedback signals~\citep{zheng2024online,zhan2026exgrpo}.
Therefore, existing methods often struggle with deficient exploration and poor sample efficiency~\citep{deng2025trial,chen2025pass}.

%%%%% $\epsilon$-greedy or Boltzmann 都是针对于 value-based RL设计的
%%%%% 目前 RL for reasoning 都是没有值函数的 policy gradient 版本
% Standard techniques like $\epsilon$-greedy or Boltzmann exploration become ineffective as they cannot efficiently navigate these enormous spaces to discover novel~\yafu{solutions?}~\citep{dann2022guarantees,cesa2017boltzmann}.

% Efficient exploration mechanisms are essential for enabling LLMs to continually self-improve their reasoning abilities through the iterative trial-and-error cycle~\citep{}.
Thus, introducing efficient exploration mechanisms becomes essential for enabling LLMs to continually self-improve their reasoning abilities through the iterative trial-and-error cycle~\citep{ladosz2022exploration,song2025outcome}.
% \citep{cui2025entropy} necessitated the entropy management mechanism for exploration and restricted the update of high-covariance tokens to alleviate the collapse of policy entropy. \citep{liu2025ettrl} used a tree-structured rollout strategy that selectively branches only high-entropy tokens to enhance the exploration–exploitation balance in test-time RL. \citep{wang2025beyond} leveraged high-entropy minority tokens as critical forks to steer the model toward diverse reasoning pathways. \citep{yao2025diversity} reformulated an entropy-based token-level diversity into the policy optimization objective to improve the model's reasoning potential.
Current LLM literature has explored several attempts based on increasing the uncertainty in the policy's action selection process~\citep{yao2025diversity}, such as alleviating policy entropy collapse~\citep{cui2025entropy}, branching only high-entropy tokens~\citep{liu2025ettrl}, and retaining policy gradient updates for high-entropy tokens only~\citep{wang2025beyond}.
These dithering strategies inject randomness into the policy's action distribution, facilitating exploration by promoting the selection of uncertain actions. 
% They typically incorporate diversity at the local token-level (i.e., action-level), without carrying out temporally-extended (or deep) exploration~\citep{osband2016deep}.
Typically, they incorporate diversity at the \textit{local} action level (i.e., token level) to help the policy escape local optima and accelerate training.
% This poses a limitation in helping the policy escape local optima, especially in complex reasoning or planning tasks.

% Diversity encourages exploration, which is essential for finding novel and effective solutions to many complex problems. Without diversity, optimization algorithms can converge prematurely, resulting in getting stuck in local optima or producing only a limited set of responses (i.e., mode collapse). 
% --ICML 2024 <Quality Diversity through Human Feedback: Towards Open-Ended Diversity-Driven Optimization>

% Provably efficient exploration schemes in RL require that exploration is directed towards potentially informative state-action pairs and consistent over multiple timesteps.

% Deep exploration means exploration which is directed over multiple time steps; it can also be called “planning to learn” or “far-sighted” exploration.

% Intuitively, an LLM capable of generating diverse responses could broaden the exploration of reasoning paths, enabling the model to avoid overfitting to narrow solution patterns in mathematical or logical tasks

% Provably efficient schemes require far-sighted exploration that is directed towards potentially informative state-action pairs and consistent over multiple timesteps~\citep{osband2016generalization}.
While action-level uncertainty estimates allow the agent to direct its exploration toward potentially informative states, they are insufficient on their own to ensure efficient exploration~\citep{osband2016deep}.
Provably efficient schemes require far-sighted, \textit{deep} exploration that is directed over multiple timesteps, inducing temporally-extended diversity at a higher level~\citep{osband2019deep}. 
Unlike local exploration, deep exploration is essential to significantly broaden the \textit{global} sequence-level diversity of reasoning pathways, stimulating the model to discover novel and effective solution patterns.
% (\textcolor{red}{TODO}: Intuitive example akin to the migratory bird case.)
%%%%% 候鸟迁徙的例子跟图1很匹配，但就是不符合事实情况，再找一个类似的例子 %%%%%
% This principle is clearly illustrated in nature, where migratory birds explore completely different continental routes across generations rather than just making random variations in individual flight movements~\citep{helbig1996genetic}.
% Recent work addresses these limitations through novel approaches to exploration. \citep{chen2025pass} introduced Pass@k Training that balances exploration-exploitation using Pass@k as the reward metric. While this approach focuses on leveraging aggregate performance across solution attempts, it does not explicitly consider the relationships between different solution paths.
% \yafu{think we may need to emphasize global diversity (compared to local acition-level div.) here.}
%For example, in daily life, people engage in global exploration when they try entirely new neighborhoods, foods, or activities, rather than making local changes to familiar routines.
% For example, migratory birds like Arctic Terns travel enormous distances, exploring multiple continents and discovering diverse food sources across various ecosystems, rather than just making random variations in individual flight movements~\citep{egevang2010tracking}.
% For example, migratory birds like Arctic Terns migrate across various continents and ecosystems to discover diverse food sources, rather than simply making local variations in their flight movements.
This principle of optimizing global diversity is essential for advancing the deep exploration capabilities of frontier RL algorithms~\citep{eysenbach2019diversity,grillotti2024quality}, showcasing remarkable efficiency for solving intricate tasks in a more human-like manner~\citep{celik2024acquiring}.
However, efficient deep exploration mechanisms remain largely underexplored in LLM reasoning.

% In contrast, intelligent creatures naturally overcome similar challenges through systematic exploration and adaptation, building diverse experiential knowledge that enhances problem-solving capabilities~\citep{zhang2021exploration,thrun1992efficient}. This accumulated exposure to varied solution pathways substantially amplifies their ability to resolve novel challenges efficiently.

% Traditional reinforcement learning addresses exploration challenges through two categories: intrinsic motivation methods and directed exploration strategies. Intrinsic motivation methods include curiosity-driven exploration~\citep{pathak2017curiosity}, which rewards agents for encountering novel states, and entropy-maximization techniques~\citep{haarnoja2018soft}, which encourage diverse policy behaviors. Directed exploration strategies encompass count-based methods~\citep{bellemare2016unifying} that track state visitation frequencies, and uncertainty-based approaches~\citep{osband2016deep} that prioritize regions of high uncertainty.

Built on these insights, we propose \textbf{DIVER} (\textbf{D}iversity-\textbf{I}ncentivized Exploration for \textbf{V}ersatil\textbf{E} \textbf{R}easoning) that emphasizes the pivotal role of global sequence-level diversity to incentivize deep exploration for versatile LLM reasoning.
% We conduct an empirical evaluation into the impact of global diversity in RLVR. We propose two easy-to-implement metrics to quantify the diversity inherent across candidate solutions in one group: Textual Diversity and Equational Diversity. Notably, the empirical evidence reveals a strong positive correlation between global diversity and reasoning performance. This finding motivates us to explicitly promote sequence-level diversity during RL training, facilitating deep exploration in a semantically structured space.
We first conduct a primary empirical study where evidence reveals a strong positive correlation between global diversity and reasoning capacity.
This finding motivates us to explicitly optimize sequence-level diversity during RL training.
To this end, we formulate the global diversity across group responses as an intrinsic reward, incentivizing deep exploration in a semantically structured space.
When incorporating this intrinsic reward, we design a potential-based reward shaping mechanism to preserve optimal policy invariance and develop simple heuristics to mitigate possible reward hacking.
Specifically, we employ two easy-to-implement metrics to quantify the diversity inherent across group responses, Textual Diversity and Equational Diversity, while in principle any other metrics are compatible with our framework.
We hope this study inspires further investigation into global diversity and incentivizes efficient deep exploration mechanisms for broadening LLM's versatile reasoning capacities.

% DIVER outperforms previous exploration RLVR methods by an average of \textbf{+1.2} points across six math reasoning benchmarks (AIME24/25, AMC, OlympiadBench, Minerva, MATH500). 
Experimental results demonstrate that DIVER consistently outperforms competitive RLVR baselines across six math reasoning benchmarks (AIME24/25, AMC, OlympiadBench, Minerva, MATH500). 
% Notably, it demonstrates stronger generalization capabilities with a \textbf{+3.2} points improvement over the GRPO baseline on out-of-domain benchmarks (ARC-c, GPQA*, MMLU-Pro). 
Notably, DIVER shows stronger generalization capabilities with a \textbf{+3.2} points improvement over the GRPO baseline on out-of-domain benchmarks (ARC-c, GPQA*, MMLU-Pro). 
To evaluate exploration effectiveness through multi-attempts, we employ the Pass@k metric, where DIVER consistently surpasses all baselines. The most substantial gain appears on AIME25, where DIVER achieves a \textbf{+6.7} points improvement in Pass@32 performance. 
% Our in-depth analysis reveals DIVER's advantage stems from its global response level diversity and exploration capability. 
Our in-depth analysis reveals that DIVER's advantage stems from its global sequence-level diversity and deep exploration capability.

% token-level <--> sequence-level
% action diversity <--> behavior diversity
% local <--> global, deep exploration

\iffalse
In summary, our main contributions are as follows:
\begin{itemize}[itemsep=0.25em, leftmargin=1.25em]\vspace{-0.2em}
% \item We propose general design principles for measuring diversity between rollouts, incorporating diversity metrics as reward shaping to guide policy models.
\item In contrast to existing methods that focus primarily on local token-level diversity, we examine the role of global sequence-level diversity in incentivizing deep exploration, revealing a positive correlation with reasoning capacity.

% \item We analyze the optimal balance between exploration and exploitation, and develop a framework that adaptively adjusts this balance.
% \item We develop a systematic scheme that injects global diversity into LLM reasoning from aspects of sampling, prompt, and optimization, with principled formulations to guarantee optimality invariance and mitigate reward hacking.
\item We introduce a reward shaping mechanism to inject global diversity incentives into LLM reasoning, with principled formulations to quantify diversity, guarantee optimal policy invariance, and mitigate reward hacking.

% \item We conduct experiments on both in-domain and out-of-domain tasks validating our approach, showing that diversity-guidance training significantly enhances model generalization.
\item \textcolor{red}{TODO: concise empirical findings, refer to LUFFY.}
\end{itemize}
\fi

\section{Related Work}
% Exploration techniques play a crucial role in solving sparse reward reinforcement learning problems. 
% These techniques can be broadly classified into three categories. 
% The first category relies on injecting stochastic noise into deterministic policies to drive exploration~\citep{lillicrap2015continuous, fujimoto2018addressing}. The second category incorporates policy entropy directly into the optimization objective, thereby controlling the stochasticity of the learned optimal policy~\citep{haarnoja2018soft, haarnoja2017reinforcement}. The third category reformulates exploration as a form of exploitation by introducing an intrinsic reward term $r^i$. This intrinsic reward, independent of environmental feedback, can take various forms such as count-based bonuses~\citep{bellemare2016unifying}, mutual information terms~\citep{houthooft2016vime}, or prediction error signals~\citep{pathak2017curiosity}. 
\paragraph{Exploration in RL.} 
Exploration techniques are key to solving high-dimensional, sparse-reward RL problems~\citep{ladosz2022exploration}.
They can be roughly categorized into three kinds: 1) injecting stochastic noise into behavior policies~\citep{lillicrap2015continuous, fujimoto2018addressing}; 2) incorporating policy entropy into the optimization objective~\citep{haarnoja2017reinforcement,haarnoja2018soft}; and 3) introducing intrinsic rewards independent of environmental feedback, such as count-based bonuses~\citep{bellemare2016unifying,wang2026thinklonger}, information gains~\citep{houthooft2016vime}, or the novelty of experience~\citep{pathak2017curiosity,burda2019exploration}. 
% Our work extends the third category of intrinsic reward shaping from traditional RL to LLM reasoning, developing a simple yet effective intrinsic diversity reward that encourages meaningful exploration across the semantic space of possible reasoning approaches.
Our work extends the third category from classical RL to LLM reasoning tasks.
Frontier RL algorithms adopt the principle of promoting global diversity to improve their deep exploration capabilities~\citep{liu2026beyond}, such as learning diverse skills in an unsupervised manner~\citep{eysenbach2019diversity}, inducing diverse behaviors with constrained optimization~\citep{grillotti2024quality}, and acquiring diverse skills using mixture-of-experts~\citep{celik2024acquiring}.
As efficient deep exploration mechanisms remain largely underexplored in LLM reasoning, our method aims to address this gap.

% \paragraph{RL for LRMs}
\paragraph{Diversity for LRMs.}
% Research on LLM diversity primarily follows two paradigms. 
% Research on promoting diversity in LLM reasoning primarily follows two paradigms.
Current literature has attempted to promote diversity in LLM reasoning by increasing the uncertainty in the action selection process. 
% The first category employs entropy-based mechanisms, where diversity is explicitly measured through entropy bonuses incorporated into objective functions~\citep{yao2025diversity}, advantage shaping~\citep{cheng2025reasoning}, or test-time scaling~\citep{liu2025ettrl}—or implicitly controlled via covariance regularization~\citep{cui2025entropy}. 
Some studies directly manage individual tokens at a micro level.
\citep{cui2025entropy} applies clip and KL penalty constraints to tokens that tend to cause entropy collapse. 
\citep{liu2025ettrl} selectively branches high-entropy tokens only to enhance exploration in test-time RL. 
\citep{wang2025beyond} leverages high-entropy minority tokens to steer the model toward diverse reasoning pathways. 
Other approaches achieve similar results by reformulating optimization objectives.
\citep{yao2025diversity} injects a token-level diversity measure into policy optimization.
\citep{cheng2025reasoning} augments the advantage function with an entropy-based term.
\citep{yu2025dapo} increases the upper bound for clipping the importance sampling ratio to emphasize low-probability tokens.
In summary, these methods typically incorporate diversity at the local action level, facilitating exploration by promoting the selection of uncertain actions.

Recently, \citep{chen2025pass} uses the Pass@k metric as the training reward, leveraging multiple candidate solutions in one trial to enhance the exploration abilities of LLMs. 
However, it does not explicitly account for global diversity, since it does not seek to optimize diversity across candidate solutions.
% The second category implicitly promotes diversity by modifying optimization objectives or constraints~\citep{chen2025pass,yu2025dapo}, without directly measuring diversity. 
% While these paradigms address exploration primarily through the lens of token-level diversity within individual reasoning paths, our work takes a fundamentally different approach. 
% We focus on promoting global semantic diversity between reasoning paths, offering a novel perspective that enhances diversity at a higher conceptual level rather than merely at the local token selection stage.
A concurrent work is~\citep{li2025jointly} that trains a partitioning classifier to measure diversity and amplifies the advantage function by the diversity assessment.
\textcolor{my_blue}{Beyond algorithmic-level exploration strategies, data-level approaches have also been proposed to improve training efficiency. Prompt selection methods~\citep{chen2025self, bae2025online, qu2025can} filter training samples based on difficulty or informativeness to enhance exploration efficiency.}
The key distinction of our method lies in how diversity is measured and how it is embedded within policy optimization.

\section{Method}
\begin{figure}[t]
    \centering
    \includegraphics[width=0.95\linewidth]{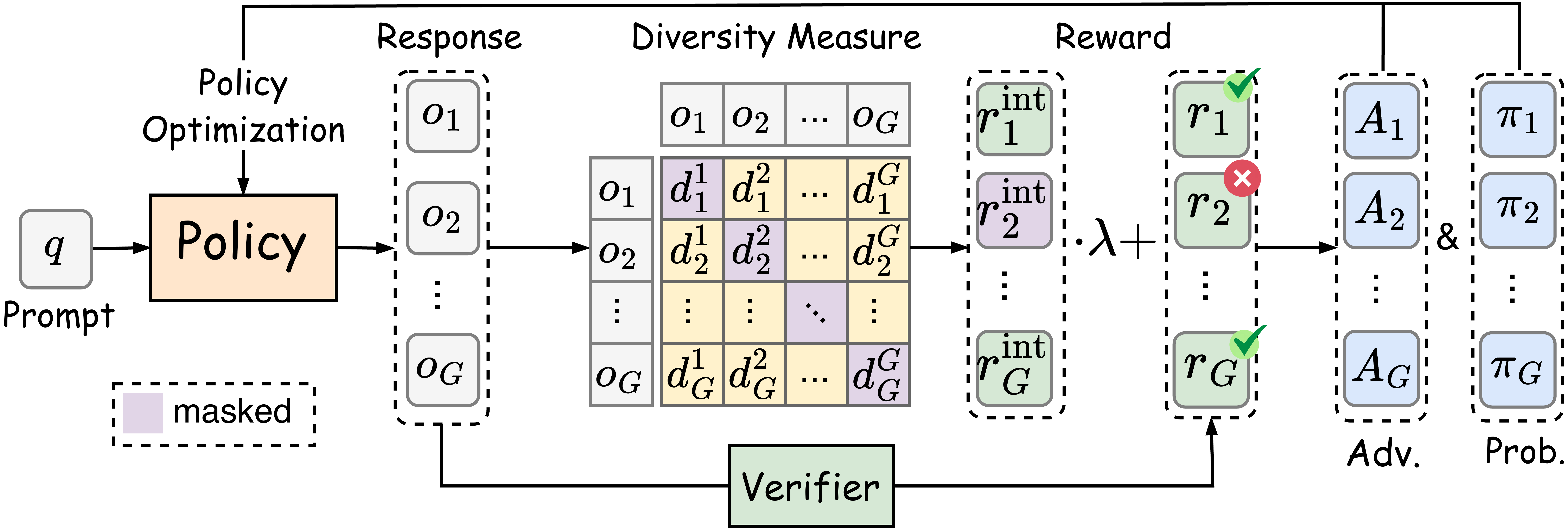}
    % \vspace{-20pt}
    \caption{Overview of DIVER where we formulate the global sequence-level diversity of response $o_i$ within a group of $G$ rollouts as an intrinsic reward $r_i^{\text{int}}$ to incentivize deep exploration.
    Diversity incentives are applied to correct solutions only to align shaping rewards with the true objective.}
    % \vspace{-15pt}
    \label{fig:framwork}
\end{figure}

In this section, we first give the problem statement where the reasoning task is formulated as an RL problem.
Then, we present a primary empirical study to show the impact of global sequence-level diversity on reasoning performance.
Finally, we introduce DIVER in detail, with principled formulations to quantify diversity, guarantee optimal policy invariance, and mitigate reward hacking.

% While GRPO provides a solid foundation, we identify that incorporating diversity is crucial for enhancing reasoning capabilities. We first analyze why diversity matters in reasoning (Section~\ref{subsec:analyze}). We then introduce our diversity-guided reward mechanism that encourages exploration of multiple reasoning trajectories (Section~\ref{subsec:div_rs}). Finally, we detail how to effectively integrate this diversity reward into the training process while avoiding reward hacking (Section~\ref{subsec:reward_hacking}).

\subsection{Problem Statement}
RL is based on the Markov decision process (MDP) formulation with a tuple $(S,A,T,R,\gamma)$, where $S/A$ is the state/action space, $T(s'|s,a):S\times A\times S\mapsto [0,\infty]$ is the transition operator that defines the probability density function of transitioning to state $s'\in S$ conditioned on taking action $a\in A$ in state $s\in S$, $R(s,a):S\times A\mapsto \mathbb{R}$ is the reward function, and $\gamma\in (0,1]$ is the discount factor.

\textbf{Group Relative Policy Optimization (GRPO).}
We build the DIVER framework upon GRPO~\citep{grpo}, a prominent RLVR baseline that incentivizes reasoning capability in LLMs using a rule-based verifiable reward function.
GRPO discards the critic model and instead computes advantages using rule-based rewards from group-level comparisons. 
For each query $q$, the policy $\pi_{\theta_{\text{old}}}$ generates $G$ candidate responses $\{o_1, ..., o_G\}$. Each response is evaluated by a binary reward function $r_i\in \{0, 1\}$ that checks whether the extracted answer matches the golden answer, yielding rewards $\{r_1, ..., r_G\}$. This verifiable reward design effectively mitigates reward hacking~\citep{hacking}, enabling robust scaling of RL training. The policy $\pi_{\theta}$ is then updated by maximizing:
\begin{equation}\label{grpo}
J_{\text{GRPO}}(\theta) = \mathbb{E}_{q\sim\mathcal{D}, \{o_i\}_{i=1}^G\sim\pi_{\theta_{old}(\cdot|q)}} \left[ \frac{1}{G} \sum_{i=1}^{G}  \sum_{t=1}^{|o_i|} \text{CLIP}(\rho_{i,t}, A_i) - \beta KL(\pi_\theta||\pi_{\text{ref}})  \right],
\end{equation}
where $\rho_{i,t} \!=\! \frac{\pi_\theta(o_{i,t}|q, o_{i,<t})}{\pi_{\theta_{\text{old}}}(o_{i,t}|q, o_{i,<t})}$ is the importance sampling ratio for the token at timestep $t$ in response $o_i$. 
The clipped objective $\text{CLIP}(\rho_{i,t}, A_i) \!=\! \min\left(\rho_{i,t} A_i, \text{clip}(\rho_{i,t}, 1-\epsilon, 1+\epsilon)A_i\right)$ ensures stable updates within the trust region~\citep{schulman2017proximal}.
The KL regularization term constrains the policy from deviating too far from a reference model $\pi_{\text{ref}}$.
The advantage $A_i$ is computed as:
\begin{equation}
A_i = \frac{r_i - \mu_r}{\sigma_r}, \quad \text{where} \quad \mu_r = \frac{1}{G}\sum_{j=1}^{G} r_j, \quad \sigma_r = \sqrt{\frac{1}{G}\sum_{j=1}^{G} (r_j - \mu_r)^2}.
\label{equ:adv}
\end{equation}

\subsection{The Impact of sequence-Level Diversity on RLVR} 
\label{subsec:analyze}

\begin{figure}[t]
    \centering
    \includegraphics[width=1\linewidth]{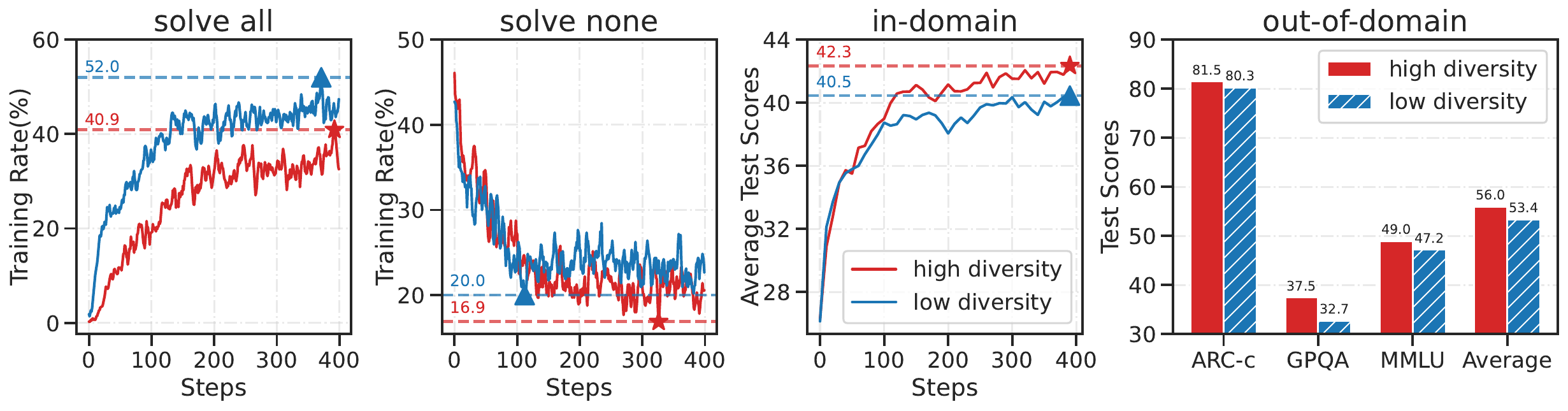}
    % \vspace{-20pt}
    \caption{Performance comparison between high-diversity (red) and low-diversity (blue) training. \texttt{solve all}: Number of samples with all rollouts correctly solved. \texttt{solve none}: Samples with no correct rollouts. \texttt{in-domain}: Average test scores across training steps for in-domain benchmarks. \texttt{out-of-domain}: Final performance for out-of-domain benchmarks.}
    % \vspace{-5pt}
    \label{fig:div_analize}
\end{figure}

While a line of recent work~\citep{liu2025ettrl,wang2025beyond} focuses on promoting local token-level diversity, we explore and analyze how global sequence-level diversity within data affects RL fine-tuning for LRMs. 
We first conduct an illustrative experiment that compares training the model using rollouts with different levels of sequence-level diversity.

Figure~\ref{fig:div_analize} shows the performance of filtering GRPO rollouts into high-diversity (red) and low-diversity (blue) subsets to train separate models, using the proposed diversity metrics (TD and ED in Sec.~\ref{sec:metrics}). The experimental details are provided in the Appendix~\ref{app:filter}.
During training, high-diversity training simultaneously yields a lower ``solve all" rate (1.1 points decrease) and a lower ``solve none" rate (3.1 points decrease).
Intuitively, the high-diversity scheme trades off a small fraction of responses to explore a wider range of reasoning patterns, enabling the search for novel breakthroughs when conventional solutions fail.
This property is particularly appealing in practice, since achieving success on complex problems carries greater value than obtaining uniformly correct answers on simple questions. 
During testing, high-diversity training achieves higher test scores, with an increase of \textbf{+1.8} points for in-domain benchmarks and an increase of \textbf{+2.6} points for out-of-domain benchmarks.
This encouraging result verifies that promoting global diversity can broaden LLM's reasoning capacity.
A noteworthy point is that the superiority of high-diversity training is more pronounced for out-of-domain tasks compared to in-domain ones.
This observation is consistent with our motivation, where emphasizing diversity enables exploring a broader spectrum of reasoning patterns and expands generalization capacity.

\subsection{Metrics for Quantifying sequence-Level Diversity}\label{sec:metrics}
The above insight highlights the necessity of promoting global diversity to enable more effective exploration for RLVR, thereby incentivizing versatile LLM reasoning.
Naturally, it is crucial to develop efficient metrics capable of quantifying sequence-level diversity in reasoning tasks.
To this end, we design two easy-to-implement metrics, Textual Diversity and Equational Diversity. 
As a general framework, our method is also compatible with any other diversity metrics.

\paragraph{Textual Diversity (TD).}
It refers to the text-level mutual dissimilarity across group responses from a given query.
We adopt the BLEU score~\citep{papineni2002bleu}, a popular automated and inexpensive metric that correlates highly with human evaluation, to quantify the text similarity, and invert it to measure dissimilarity.
In essence, BLEU measures similarity by calculating the overlap of n-grams (word sequences) between a candidate text and the reference text, with higher scores indicating greater similarity. 
Appendix~\ref{app:bleu} presents the detailed definition.

For each candidate response $o_i$ in the group $\{o_1,...,o_G\}$, we calculate its BLEU similarity to every other response and derive its TD within the group as
\begin{equation}
\mathrm{TD}(o_i)
= \frac{1}{G-1}\sum\nolimits_{j\in[G]\setminus\{i\}}
\Bigl(1-\mathrm{BLEU}(o_i,o_j)\Bigr).   
\end{equation}

\paragraph{Equational Diversity (ED).}
It refers to the differences in the formulas contained across group responses to a given query, aiming to characterize underlying reasoning patterns for mathematical tasks.
A math problem often admits multiple solutions with distinct formulaic patterns, and exploring diverse problem-solving approaches can increase the likelihood of finding the correct answer.

% For each prompt with $N$ sampled responses, denote the $i$-th response by $y_i$ and let $\mathcal{F}(y_i)$ be the set of formulas extracted from it. Collect formulas from the other responses as
Let $\mathcal{F}(o_i)$ denote the set of formulas extracted from response $o_i$, and  $\mathcal{F}_{-i}$ denote the set of formulas collected from other responses in the same group as
\begin{equation}
\mathcal{F}_{-i}=\bigcup\nolimits_{j\in[G]\setminus\{i\}}\mathcal{F}(o_j),
\quad [G]=\{1,\ldots,G\}.
\end{equation}
% The per-response diversity score is
Then, the per-response equational diversity is defined as the ratio of unique formulas in response $o_i$ to  the total number of formulas it contains:
% \begin{equation}
% \mathrm{ED}(o_i)=
% \begin{cases}
% \dfrac{\bigl|\mathcal{F}(o_i)\setminus\mathcal{F}_{-i}\bigr|}{\bigl|\mathcal{F}(o_i)\bigr|}, & |\mathcal{F}(o_i)|>0,\\[6pt]
% 0, & \text{otherwise.}
% \end{cases}
% \end{equation}
\begin{equation}
\mathrm{ED}(o_i)= \dfrac{\bigl|\mathcal{F}(o_i)\setminus\mathcal{F}_{-i}\bigr|}{\bigl|\mathcal{F}(o_i)\bigr|},~~~~\text{if}~~|\mathcal{F}(o_i)|>0;~~~~\text{or}~~0,~~\text{otherwise}.
\end{equation}

% \subsection{Diversity-Guided Reward Shaping}
\subsection{Promoting Global Diversity for Deep Exploration}
\label{subsec:div_rs}
% In this section, we present \textbf{DIVER} (\textbf{D}iversity-\textbf{I}ncentivized Exploration for \textbf{V}ersatil\textbf{E} \textbf{R}easoning), an innovative framework that injects global diversity into RLVR, incentivizing deep exploration to broaden LLM's versatile reasoning capacities.
% The above finding motivates 
The findings in Sec.~\ref{subsec:analyze} motivate us to explicitly optimize global diversity during RL training.
% Figure~\ref{fig:framwork} illustrates DIVER's reward shaping mechanism, where we formulate the diversity across group responses as an intrinsic reward.
As shown in Figure~\ref{fig:framwork}, we formulate the global diversity as an intrinsic reward to incentivize deep exploration in a semantically structured space.
For a group of responses $\{o_1,...,o_G\}$, we calculate the pairwise dissimilarity using the proposed diversity metrics in Sec.~\ref{sec:metrics}, yielding a $G\!\times\! G$ matrix $D$ where each element $d_i^j$ denotes the diversity between responses $o_i$ and $o_j$.
Then, the diversity of response $o_i$ within the group is calculated by simply averaging its dissimilarities to all others as $d(o_i)\!=\!\frac{1}{G-1}\sum_{j\ne i}^{G}d_i^j$, i.e., averaging across the corresponding row in the diversity matrix $D$.

To promote global diversity in RLVR, a natural option is to directly supply the quantified diversity as an additional reward to guide the learning process.
However, this naive shaping can change the optimal policy and mislead the agent into learning suboptimal policies~\citep{ng1999policy}.
Hence, we adopt a potential-based reward shaping scheme to preserve optimal policy invariance when incorporating the intrinsic reward~\citep{wang2023efficient,muller2025improving}. 
We formulate the \textbf{intrinsic reward} $R_{\text{int}}$ as the difference between the sequence-level diversities $d(\cdot)$ of adjacent states as 
\begin{equation}\label{eq:rint}
    R_{\text{int}}(s_t,a_t,s_{t+1})=\gamma d(s_{t+1}) - d(s_t),
\end{equation}
where $d(\cdot)$ is the exactly the potential function over states $s\in S$.
In the LLM setting of Eq.~(\ref{grpo}), the states and action within a given query-response pair are defined as $s_t:=[q,o_{i,\le t}]$, $a_t:=o_{i,t+1}$, and $s_{t+1}:=[q,o_{i,\le t+1}]$.
% To optimize for global diversity, we set the potential function $\Phi(s)$ as the state's sequence-level diversity.
Then, the specific intrinsic reward becomes
\begin{equation}
    R_{\text{int}}([q, o_{i,\le t}],o_{i,t+1}, [q,o_{i,\le t+1}])=\gamma d([q, o_{i,\le t}]) - d([q,o_{i,\le t+1}]), \quad i=1,...,G.
\end{equation}
Since GRPO inherits the PPO principle~\citep{schulman2017proximal} that derives policy gradients at the sequence level, the intrinsic reward for a complete query-response pair is calculated as
\begin{equation}\label{eq:rint2}
\begin{aligned}
    R_{\text{int}}\Bigl([q,o_i]\Bigr) & = \sum\nolimits_{t=0}^{T-1}\gamma^tR_{\text{int}}\Bigl([q, o_{i,\le t}],o_{i,t+1}, [q,o_{i,\le t+1}]\Bigr) \\
    % & = \gamma d([q,o_i])-d(q) + (\gamma-1)\sum_{t=1}^{T-1}d([q, o_{i,\le t}]) \\
    % & = \gamma d([q,o_i]),
    & = \sum\nolimits_{t=0}^{T-1}\gamma^t[\gamma d([q,o_{i,\le t+1}]) - d([q,o_{i,\le t}])] \\
    & = \gamma^Td([q,o_{i,\le T}]) - d(q)\\
    & = \gamma^Td([q,o_i]),
\end{aligned}
\end{equation}
where $T$ denotes the terminal step, and $d([q,o_i])$ is the global diversity of response $o_i$, which can be calculated by metrics in Sec.~\ref{sec:metrics}.
The diversity of a constant query $q$ is zero, i.e., $d(q)=0$.
% Since we focus on the global diversity for an entire sentence, we set the diversity of any subset of the response to zero, i.e., $d([q,o_{i,\le t}])=0,\forall t<T$.

This diversity-incentivized intrinsic reward complements the traditional rule-based assessment reward $R(\cdot)$. 
By incorporating the diversity measure, we design a versatile evaluation system that values both correctness and solution diversity, yielding the new reward function $R'(\cdot)$ as 
\begin{equation}\label{eq:aug_reward}
    R'([q,o_i]) = R([q,o_i]) + \lambda\cdot R_{\text{int}}([q,o_i]),
\end{equation}
where $\lambda$ is the shaping ratio that balances between accuracy and diversity. 
Maximizing the intrinsic reward incentivizes the model to explore diverse reasoning pathways at the sequence level, facilitating the discovery of novel and effective solutions to complex problems.
Finally, we substitute the augmented reward function $r_i'\!=\!R'([q,o_i])$ for the original reward $r_i\!=\!R([q,o_i])$ in Eq.~(\ref{equ:adv}) to calculate the advantage function under the GRPO framework.

% This diversity-based reward $\hat{r}_i$ serves as a reward shaping mechanism that complements the traditional rule-based assessment. By incorporating this diversity measure into our reward function, we create a comprehensive evaluation framework that values both correctness and diversity:
% \begin{equation}
% R(q, y_i) = r(y_i) + \lambda \cdot \hat{r}_i
% \end{equation}
% Here, $r(y_i)$ represents the rule-based reward assessing the correctness of the solution provided in response $y_i$ to prompt $q$, $D(y_i)$ quantifies the diversity contribution, and $\lambda$ controls the balance between accuracy and diversity. This formulation enables simultaneous optimization of solution accuracy and reasoning diversity, significantly enhancing the model's adaptability across varied challenges.

% To integrate our diversity-guided reward mechanism into the GRPO framework, we modify the advantage calculation in Eq.~\ref{equ:adv} to incorporate our enhanced reward function:
% \begin{equation}
% A_i = \frac{R(q, y_i) - \mu_R}{\sigma_R}, \quad \text{with} \quad \mu_R = \frac{1}{G}\sum_{j=1}^{G}R(q, y_j), \quad \sigma_R = \sqrt{\frac{1}{G}\sum_{j=1}^{G}(R(q, y_j) - \mu_R)^2}
% \end{equation}
% \yafu{no need to rewrite the adv. computation, refer to Eq.2 is clear enough.}

% Importantly, our diversity-guided reward shaping maintains policy invariance, a critical property in reinforcement learning. We can demonstrate that the addition of our diversity term preserves the optimal policy structure while enhancing exploration of diverse reasoning paths. 
When including the intrinsic reward, we will transform the original MDP $M\!=\!(S,A,T,R,\gamma)$ to a new one $M'\!=\!(S,A,T,R',\gamma)$, where $R'\!=\!R+\lambda R_{\text{int}}$.
Since we are learning a policy for the transformed MDP $M'$ in the hope of using it in the original one $M$, it is essential to ensure that this transformation does not mislead the agent into learning suboptimal policies.
Theorem~\ref{theo:shaping} guarantees the optimal policy invariance when incorporating global diversity as an intrinsic reward, validating the effectiveness of our reward shaping mechanism.
Appendix~\ref{app:reward_shaping} presents the detailed proof.

\begin{theorem}[Optimal Policy Invariance]\label{theo:shaping}
    Let $M\!=\!(S,A,T,R,\gamma)$ denote the MDP for the LLM reasoning task.
    $d(\cdot)\!:S\mapsto \mathbb{R}$ is a real-valued function that computes the sequence-level diversity $d(s)$ of the state $s$ within a group of rollouts.
    We formulate $R_{\text{int}}(\cdot)\!:S\!\times\! A\!\times\! S\mapsto \mathbb{R}$ as an intrinsic reward function that is the difference between sentence diversities of two adjacent states, such that for all $s\in S, a\in A, s'\in S$, $R_{\text{int}}(s,a,s')=\gamma d(s')-d(s)$. 
    Then, with any constant balancing ratio $\lambda$, every optimal policy in the transformed MDP $M'\!=\!(S,A,T,R+\lambda R_{\text{int}},\gamma)$ will also be an optimal policy in $M$, and vice versa.
\end{theorem}

\textcolor{my_blue}{The ingenuity of our reward-shaping design lies in that by setting the shaping reward as the difference between the diversities of adjacent states as in Eq.~(\ref{eq:rint}), the intrinsic reward for a complete query-response pair is derived as the diversity of the final response as in Eq.~(\ref{eq:rint2}). 
This elegant formulation avoids the need to calculate the diversity of any intermediate sentences, which saves a significant amount of computation while ensuring optimal policy invariance.}

% \subsection{Preventing Reward Hacking in Reward Shaping}
\subsection{Mitigating Reward Hacking}
\label{subsec:reward_hacking}
% \yafu{Mitigating Reward Hacking}

% The reward function and transition dynamics collectively define the world model or task specification in reinforcement learning. 
% \yafu{``While'' is weird here. Consider ``In addition to primary function xxx, DIVER integrates diversity reward via reward shaping xxx that xxx''}
% While the primary reward function $r(y_i)$ encodes the fundamental objective (correctness of reasoning in our case), reward shaping techniques like our diversity incentive $\hat{r}_i$ serve to guide exploration and improve learning efficiency.
% However, these reward modifications introduce the risk of reward hacking, a phenomenon where the model exploits the reward structure in unintended ways, optimizing for the shaped reward rather than the true objective. 
Including an additional shaping reward could increase the risk of reward hacking, a phenomenon where an RL agent exploits flaws or ambiguities in the reward function to achieve high rewards without genuinely solving the intended task~\citep{pan2022effects}.
This is particularly concerning in language models, where the complex nature of reasoning tasks makes reward functions susceptible to biased exploitation~\citep{liu2025rrm}.
Although DIVER preserves optimal policy invariance after reward shaping, the model may still over-exploit intrinsic rewards and neglect the primary objective during training.
Since the primary reward for reasoning correctness is sparse and difficult to attain, obtaining the reward for diversity is considerably easier, especially when addressing hard problems.
% To mitigate reward hacking in DIVER, we implement several methodological safeguards. 
To this end, we design simple heuristics to mitigate the potential risk of reward hacking as follows.
% \yafu{Refer to the subsequent ablation study sections.}
Ablation study in Sec.~\ref{sec:analyses} verifies the successful mitigation of possible reward hacking.

% \paragraph{Well-Designed Shaping Rewards with Balanced Weighting}
\paragraph{Balanced Shaping.}
% Our diversity calculation is performed on complete responses within the rollout groups. By computing policy advantages within prompt-specific response distributions, we establish a relative optimization framework that reduces systematic exploitation opportunities.  
% The coefficient $\lambda$ in $R(q, y_i) = r(y_i) + \lambda \cdot \hat{r}_i$ is systematically tuned to balance diversity and task performance, while we apply clipping to the diversity metric $\hat{r}_i=\text{clip}(\hat{r}_i, 0, \tau)$, where $\tau$ is a predetermined upper bound that prevents excessive diversity rewards.
We clip the diversity reward to be $r_i^{\text{int}}\!=\!\text{clip}(r_i^{\text{int}};0,\sigma)$, where $r_i^{\text{int}}\!=\!R_{\text{int}}([q,o_i])$ and $\sigma$ is a predetermined upper bound that prevents the model from excessively exploiting the shaping reward.
% Additionally, we carefully adjust the timing of when the shaping reward is applied during training, gradually reducing its influence at appropriate timesteps. These calibrated mechanisms work in concert to ensure diversity enhancement without compromising the primary learning objective. 
Moreover, we gradually reduce the balancing ratio $\lambda$ during training.
Akin to the classic exploration-exploitation tradeoff in RL philosophy~\citep{sutton2018reinforcement}, we prefer exploring diverse solutions early and tend to exploit accumulated knowledge later.

% \paragraph{Conditional Reward Application with Priority Ordering}
\paragraph{Conditional Shaping.}
% To further safeguard against reward hacking, our diversity reward mechanism employs a clear priority structure. 
% Diversity bonuses are only applied to responses that are correctly solved within the rollout group, which can be formulated as $R(q, y_i) = r(y_i) + \lambda \cdot \hat{r}_i \cdot I(r_i)$, where $I(r_i)$ is an indicator function that equals 1 if the response is correct and 0 otherwise, as shown in Fig~\ref{fig:framwork}. 
We only include the shaping reward to correct responses within the group as $r_i'\!=\!r_i+\lambda \cdot r_{\text{int}}^i\cdot I(r_i)$, where $I(r_i)$ is an indicator function that equals 1 if the response is correct and 0 otherwise. 
% This conditional application ensures that diversity enhancement only rewards genuinely correct solutions, preventing the model from optimizing for diversity at the expense of correctness. 
This conditional shaping ensures that the diversity incentive only rewards genuinely correct solutions, preventing the model from trading off correctness for diversity.
% By prioritizing correctness before considering diversity, we maintain the integrity of the primary learning objective while allowing the exploration incentive to differentiate between multiple valid solutions. 
% This design principle effectively aligns the shaped reward with the true objective, addressing potential reward hacking concerns while promoting valuable exploration across the solution space.
The design principle effectively aligns shaping rewards with the true objective, addressing potential reward hacking concerns while promoting valuable, diversified exploration across the solution space.

\section{Experiments}
% We evaluate DIVER on both in-distribution tasks and out-of-distribution tasks(Sec.~\ref{sec:main_results}), test the Pass@k metric to measure the boundaries of model generalization capabilities~\ref{sec:passk}, and provide detailed analyses~\ref{sec:analyses}.

% We evaluate DIVER on both in-distribution and out-of-distribution tasks (Sec.~\ref{sec:main_results}), test the Pass@k metric to measure the boundaries of model generalization capabilities (Sec.~\ref{sec:passk}), and provide detailed analyses (Sec.~\ref{sec:analyses}).
% With extensive experiments, we aim to answer the following questions: 
% i) Can DIVER improve performance in both in-domain and out-of-domain benchmarks?
% ii) Is it possible to enhance Pass@k capability without compromising Pass@1 performance?
% iii) Does DIVER successfully mitigate reward hacking when incorporating diversity rewards?
We comprehensively evaluate and analyze our method to answer the following research questions: i) Can DIVER improve performance while maintaining effective global exploration and reliably extending to other models? ii) Can DIVER achieve an effective and broader exploration scope that unlocks enhanced reasoning capacity? iii) What is the appropriate configuration of DIVER for balancing diversity, reward stability, and exploration horizons?

% \begin{itemize}[itemsep=0.1em, leftmargin=1.25em]\vspace{-1em}
%     \item Can DIVER improve performance in both in-domain and out-of-domain benchmarks?
%     \item Is it possible to enhance Pass@k capability without compromising Pass@1 performance?
%     \item Does DIVER successfully mitigate reward hacking when incorporating diversity rewards?
% \end{itemize}

\subsection{Experimental Settings}
\label{sec:ex_setup}
% \yafu{add one para: Training to introduce training details and model setup.}
\paragraph{Datasets and Evaluation.} Our training data is a subset of OpenR1-Math-220k~\citep{openr1}, with prompts collected from NuminaMath 1.5~\citep{numina_math_datasets} following the LUFFY~\citep{yan2025learning}.
% We follow the LUFFY~\citep{luffy}\footnote{\url{https://huggingface.co/datasets/Elliott/Openr1-Math-46k-8192}} dataset construction methodology but differ in that we do not incorporate off-policy reasoning traces, as ours is a purely on-policy approach.
% \paragraph{Evaluation} 
We evaluate on six mathematical reasoning benchmarks: AIME24/25, AMC~\citep{li2024numinamath}, Minerva~\citep{dataset_minerva}, OlympiadBench~\citep{dataset_olympiad}, and MATH-500~\citep{dataset_math}. Main results report Avg@32 for the smaller test sets (AIME24/25, AMC), and Pass@1 for others. For cross-domain, we test on ARC-c~\citep{arc}, GPQA-diamond (GPQA$^*$)~\citep{gpqa}, and MMLU-Pro~\citep{mmlu_pro}. 
\paragraph{Baselines and Training.} We compare DIVER against two categories of baselines: 1) Established RLVR methods: \texttt{SimpleRL-Zoo}\citep{simplerlzoo}, \texttt{OpenReasoner-Zero}\citep{orz}, and \texttt{PRIME-Zero}~\citep{prime}. 2) Our reproduction of exploration RL methods: \texttt{GRPO w/ Clip-higher}\citep{yu2025dapo}, which modifies clip ratio to encourage exploration; \texttt{Entropy-RL}\citep{cui2025entropy}, which addresses policy entropy collapse through covariance-based techniques; and \texttt{Pass@k Training}~\citep{chen2025pass}, which uses Pass@k as the reward to adaptively balance exploration and exploitation.
% \paragraph{RL Practice} 
% We remove the KL loss term by setting $\beta=0$ and set higher clip as 0.28 following the GPPO w/ Clip-higher. Our sample batch size is 128, and the update batch size is 32. We use 8 rollouts per prompt. 
We set $\beta=0$ to remove the KL loss term and use 0.28 for higher clip following GPPO w/ Clip-higher.
Sample batch size is 128, update batch size is 32, with 8 rollouts per prompt.
All experimental details are documented in Appendix~\ref{appendix:ex_details}.

\subsection{Main Results}\label{sec:main_results}
\begin{table*}[t]
\centering
% \vspace{-20pt}
\caption{Performance comparison across in-domain and out-of-domain tasks based on \texttt{Qwen2.5-Math-7B}. Best results in \textbf{bold} and second best \underline{underlined}. DIVER-TD and DIVER-ED represent our approach implemented with Textual Diversity and Equational Diversity, respectively. \textcolor{my_blue}{DIVER-MIX combines both metrics during training to achieve optimal diversity.}}
\label{tab:main_results}
\setlength{\tabcolsep}{2.5pt}  
\renewcommand{\arraystretch}{1.3} 
\resizebox{\textwidth}{!}{%
\begin{tabular}{lccccc>{\columncolor{yellow!12}}c|ccc>{\columncolor{cyan!10}}c}
\toprule
\multirow{2}{*}{\textbf{Model}} & \multicolumn{6}{c}{\textbf{In-Domain Performance}} & \multicolumn{4}{c}{\textbf{Out-of-Domain Performance}} \\
\cmidrule(lr){2-7} \cmidrule(lr){8-11}
 & \textbf{AIME 24/25} & \textbf{AMC} & \textbf{MATH-500} & \textbf{Minerva} & \textbf{Olympiad} & \textbf{Avg.} & \textbf{ARC-c} & \textbf{GPQA}$^{*}$ & \textbf{MMLU-Pro} & \textbf{Avg.} \\
\midrule
% \multicolumn{11}{c}{Backbone} \\
% \midrule
Qwen2.5-Math-7B     
  & 11.8/6.3    & 43.1 & 56.8 & 16.9  & 25.4 & 26.7      & 38.1 & 12.2 & 31.5 & 27.3     \\
% GRPO     
%   & 21.0/14.4    & 55.8 & \underline{83.0} & \underline{33.5}  & 40.4 & 41.3      & 81.8 & 39.8 & 48.5 & 56.7     \\
\midrule
\normalrow
\multicolumn{11}{c}{Previous RLVR methods} \\
\midrule
SimpleRL-Zoo             
  & \textbf{25.2}/12.0    & 57.6 & 76.2 & 27.2 & 41.0 & 39.9      & 22.0 & 20.4 & 32.5 & 25.0    \\
OpenReasoner-Zero                
  & 16.5/15.0   & 52.1 & 82.4 & \underline{33.1} & \textbf{47.1} & 41.0      & 66.2 & 29.8 & \textbf{58.7} & 51.6 \\
PRIME-Zero                   
  & 17.0/12.8   & 54.0 & 81.4 & \textbf{39.0 }& 40.3 & 40.7      & 73.3 & 18.2 & 32.7 & 41.4     \\
\midrule
\normalrow
\multicolumn{11}{c}{Exploration RL Methods} \\
\midrule
GRPO w/ Clip-higher & 18.9/\underline{16.4} & 57.3 & 81.2 & 28.7 & 41.5 & 40.7 & {82.1} & 36.2 & 47.2 & 55.2     \\
Entropy-RL & 23.6/12.8 & 58.4 & 82.8 & 31.6 & 41.5 & 41.8 & 80.7 & 38.8 & 48.4 & 56.0     \\
% Pass@k Training & 20.9/15.7 & 52.3 & 83.8 & 32.7 & 43.8 & 41.5  & 79.3 & 37.8 & 49.0 & 55.3     \\
\textcolor{my_blue}{Pass@k Training} & \textcolor{my_blue}{20.9/15.7} & \textcolor{my_blue}{52.3} & \textcolor{my_blue}{83.8} & \textcolor{my_blue}{32.7} & \textcolor{my_blue}{43.8} & \textcolor{my_blue}{41.5} & \textcolor{my_blue}{79.3} & \textcolor{my_blue}{37.8} & \textcolor{my_blue}{49.0} & \textcolor{my_blue}{55.3} \\
\midrule
\rowhighlight
\multicolumn{11}{c}{\textbf{Our Methods}} \\
\midrule
DIVER-TD                                        
   &22.5/\textbf{16.9} & 59.4 & 82.2 & 27.9 & 44.7 & \underline{42.3}     & \underline{83.4} & \textbf{42.3} & 49.5 & \underline{58.4}     \\
DIVER-ED                                        
  & 20.9/15.7 & \underline{59.7} & \underline{84.0} & 31.6 & \underline{46.1} & \underline{43.0} & \underline{83.4} & 36.2 & {49.9} & 56.5 \\
% DIVER-MIX                                        
%   & \underline{23.8}/\underline{16.4} & \textbf{60.9} & \textbf{84.4} & 29.4 & 44.0 & \textbf{43.1} & \textbf{84.1} & \underline{41.3} & \underline{51.0} & \textbf{58.8} \\
\textcolor{my_blue}{DIVER-MIX} & \textcolor{my_blue}{\underline{23.8}/\underline{16.4}} & \textcolor{my_blue}{\textbf{60.9}} & \textcolor{my_blue}{\textbf{84.4}} & \textcolor{my_blue}{29.4} & \textcolor{my_blue}{44.0} & \textcolor{my_blue}{\textbf{43.1}} & \textcolor{my_blue}{\textbf{84.1}} & \textcolor{my_blue}{\underline{41.3}} & \textcolor{my_blue}{\underline{51.0}} & \textcolor{my_blue}{\textbf{58.8}} \\

\bottomrule
% \vspace{-35pt}
\end{tabular}
}
\end{table*}
\paragraph{Reasoning Performance on Qwen2.5-Math-7B.}
% Table~\ref{tab:main_results} illustrates DIVER's evaluation results across both in-domain and out-of-domain benchmarks compared to established RLVR methods (SimpleRL-Zoo, OpenReasoner-Zero, and PRIME-Zero). 
Table~\ref{tab:main_results} illustrates DIVER's evaluation results compared to established RLVR methods (SimpleRL-Zoo, OpenReasoner-Zero, and PRIME-Zero). 
All implementations based on Qwen2.5-Math-7B. DIVER with Textual Diversity (TD) and Equational Diversity (ED) achieves average scores of \textbf{42.3} and \textbf{43.0} on six mathematical benchmarks, outperforming OpenReasoner-Zero by \textbf{+2.0} points. On challenging out-of-domain tasks, DIVER reaches \textbf{58.4} average score, surpassing OpenReasoner-Zero by \textbf{+6.8} points, with notable gains on ARC-c (\textbf{+10.1}) and GPQA (\textbf{+12.5}). These results confirm that encouraging diverse reasoning paths at the sequence level enhances model generalization significantly.

\paragraph{Comparison with  Exploration RL Methods.}
% We evaluate DIVER against several representative exploration RLVR methods for comparison.
% GPRO w/ Clip-higher implements \textit{undirected exploration} by increasing the clip constraint, Entropy-RL enhances \textit{action-level exploration} through entropy control, Pass@k Training promotes \textit{exploration indirectly} by reevaluating rewards through \textit{within-group} rollout pass@k mechanisms. 
We evaluate DIVER against representative exploration RLVR methods: GPRO w/ Clip-higher (\textit{undirected exploration}), Entropy-RL (\textit{action-level exploration}), and Pass@k Training (\textit{within-group bootstrap sampling}). Experimental results demonstrate that DIVER outperforms the best exploration method, Entropy-RL, by \textbf{+1.2} points on average across in-domain tasks. This improvement is particularly pronounced on challenging benchmarks, with a \textbf{+4.6} point advantage on OlympiadBench. Notably, on out-of-domain tasks, both DIVER and the global exploration approach Pass@k Training exhibit superior generalization compared to local exploration techniques. DIVER surpasses the best local exploration method, Entropy-RL, by \textbf{+2.4} points on average. This advantage is especially evident on benchmarks that assess conversational capabilities, such as GPQA, where DIVER outperforms Entropy-RL by \textbf{+3.5} points. The relatively weaker performance of GPRO w/ Clip-higher indicates that merely increasing constraints without strategic direction fails to fundamentally enhance model generalization capabilities.

\paragraph{Training Dynamics of Exploration Methods.}
\begin{figure}[t]
    \centering
    \includegraphics[width=1\linewidth]{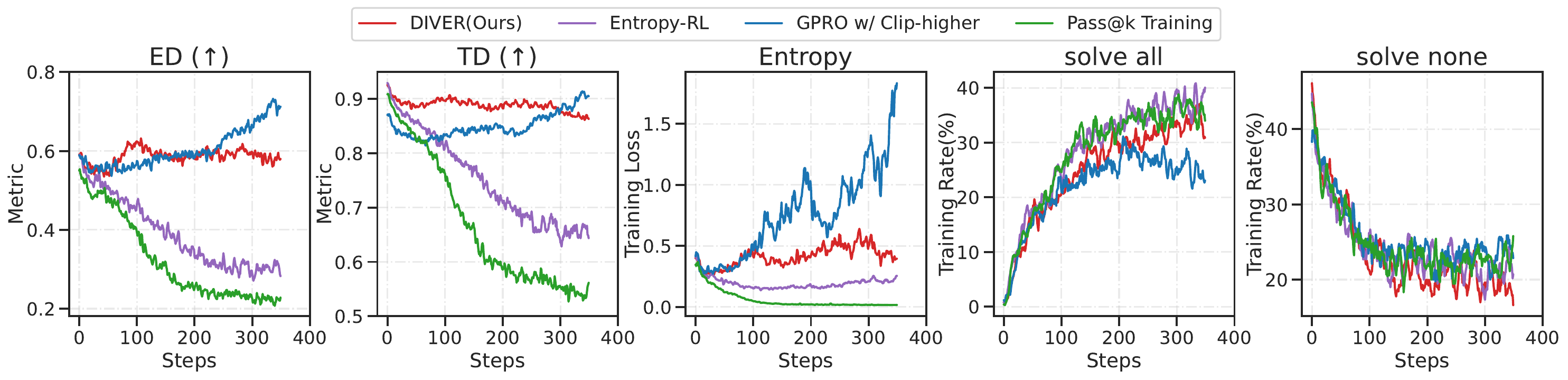}
    % \vspace{-20pt}
    \caption{Training dynamics comparison with other exploration method across different metrics. $\uparrow$ indicates metrics where higher values are more diverse for ED and TD.
    % DIVER demonstrates balanced performance in ED, TD, and entropy metrics.
    % While its ``solve all" rate increases somewhat more gradually, it achieves lower ``solve none" rates, highlighting the model's potential.
    }
    \label{fig:metric}
    \vspace{-5pt}
\end{figure}
% Figure.~\ref{fig:metric} illustrates the training dynamics of DIVER compared to other exploration methods across various metrics. Examining the exploration-focused baselines, we observe distinctive patterns. For diversity metrics (ED and TD), Pass@k Training and Entropy-RL show significant decreases over time, indicating diminishing exploration as training progresses. GPRO w/ Clip-higher maintains higher ED/TD values but at the cost of excessive entropy, which increases dramatically in later training stages.
Figure.~\ref{fig:metric} illustrates the various metrics of DIVER compared to other exploration methods. 
We aim to maintain sufficient exploration (high ED and TD) while preventing excessive entropy growth that may lead to model collapse. 
% With this logic in mind, we observe that Pass@k Training and Entropy-RL show significant decreases over time, indicating diminishing exploration. GPRO w/ Clip-higher maintains higher diversity but suffers from excessive entropy growth in later stages.
With this logic in mind, we observe that Pass@k Training and Entropy-RL diversity metrics decline over time, indicating reduced exploration, while GPRO w/ Clip-higher maintains diversity but experiences problematic entropy increases later.
% In contrast, DIVER achieves the optimal balance by preserving high diversity values while maintaining consistently low entropy, demonstrating controlled exploration without unnecessary randomness.
In contrast, DIVER achieves optimal balance with high diversity and consistently reasonable entropy levels, enabling controlled exploration without excess randomness.
% The rightmost plots further highlight these differences, while DIVER's ``solve all" rate increases more gradually than some competitors, it achieves lower ``solve none" rates by the end of training. This balanced approach demonstrates DIVER's effective exploration without sacrificing solution quality.
Notably, DIVER's ``solve all" rate grows more gradually, but its lower ``solve none" rate in training demonstrates effective exploration without compromising solution quality.
% This balanced approach shows DIVER's ability to maintain meaningful diversity without sacrificing solution quality
% , unlike the undirected exploration of GPRO w/ Clip-higher which struggles with poorer solve rates despite high diversity metrics.

\paragraph{Extending DIVER to Different Models.}
We further explore DIVER's adaptability across various language models, including \textit{small, weak} or \textit{different architecture} models. As shown in Figure.\ref{fig:model_scacle} and Table~\ref{tab:main_other_models}, DIVER maintains effectiveness across different model backbones including Qwen2.5-Math-1.5B, Qwen2.5-7B-Base, and LLaMA-3.1-8B-Insturct. DIVER consistently outperforms baselines, improving over GRPO w/ Clip-higher by +1.7, +1.7, and +1.5 points on in-domain tasks respectively, with even larger gains of +1.5, +1.3, and +1.9 points on out-of-domain tasks. \textcolor{my_blue}{Additionally, experiments with DeepSeek-R1-Distill-Qwen-7B, a model generating longer responses (2500-3500 tokens), show DIVER outperforming GPRO w/ Clip-higher by +1.6 points in-domain and +3.4 points out-of-domain, demonstrating its effectiveness with longer reasoning horizons.}

\begin{wrapfigure}{t}{0.43\textwidth}
  \centering 
  \vspace{10pt}
  \includegraphics[width=\linewidth]{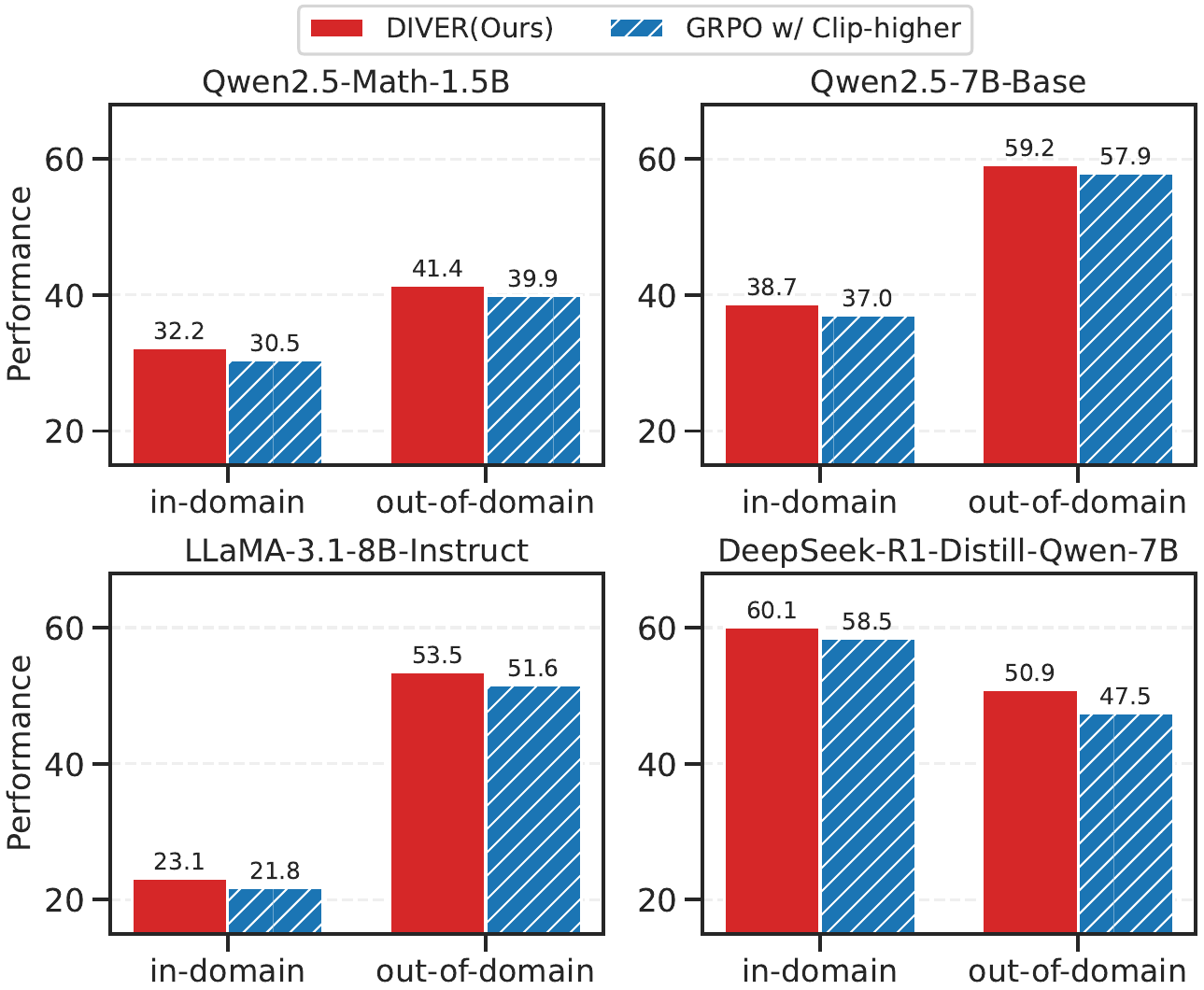}
  \caption{\textcolor{my_blue}{Average scores across in-domain and out-of-domain tasks with different models. Complete results in Table~\ref{tab:main_other_models}.}}
  \label{fig:model_scacle}
  \vspace{-5pt}
\end{wrapfigure}
% Across both in-domain and out-of-domain tasks, DIVER consistently outperforms the baselines. Specifically, DIVER improves performance over GRPO w/ Clip-higher by +1.7 points on Qwen2.5-Math-1.5B, +1.7 points on Qwen2.5-7B-Base, and +1.5 points on LLaMA-3.1-8B-Instruct for in-domain tasks. The improvements are even more pronounced for out-of-domain generalization, with gains of +1.5, +1.3, and +1.9 points, respectively. 

% \subsection{Exploration Boundary and Reasoning Capability}
% \subsection{\yafu{Exploration Scope and Reasoning Capacity}}
\subsection{Exploration Scope and Reasoning Capacity}
\label{sec:passk}
\begin{figure}[t]
    \vspace{-10pt}
    \centering
    \includegraphics[width=0.98\linewidth]{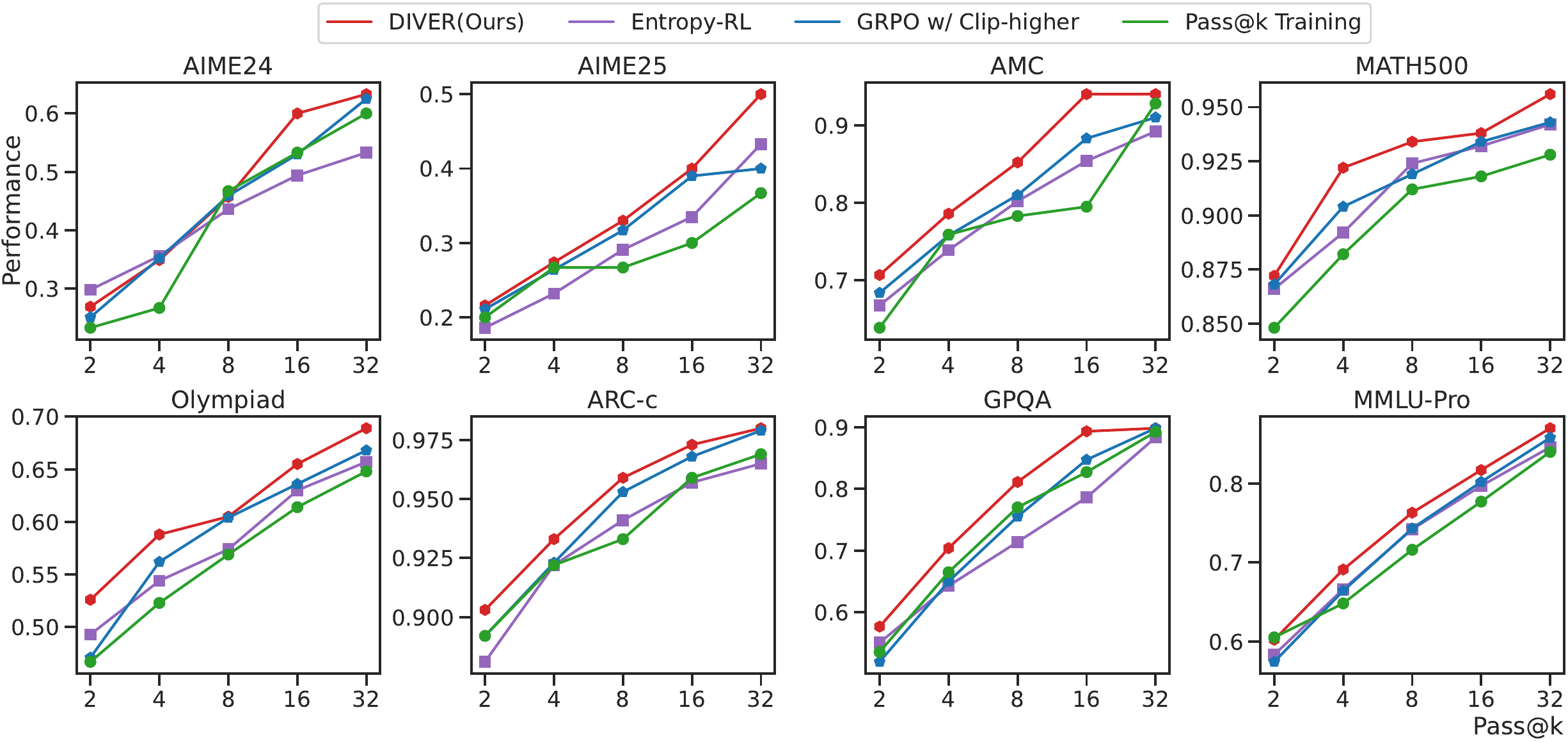}
    % \vspace{-15pt}
    \caption{Comparison of different Pass@k performance across in-domain and out-of-domain benchmarks based on \texttt{Qwen2.5-Math-7B}. Complete results are available in Table~\ref{tab:pass@k}.}
    % DIVER demonstrates steeper performance gains across nearly all benchmarks compared to other exploration methods. 
    % \vspace{-10pt}
    \label{fig:passk}
\end{figure}

% \paragraph{Revealing Exploration Boundary and Reasoning Capability}
% Unlike Pass@1 which evaluates single-attempt accuracy, Pass@k measures the capability and effectiveness of the reasoning model by determining if at least one solution among k attempts is correct. 
% We compare with the exploration-based methods on pass@k,which measures the upper-bound of model reasoning capabilities~\citep{cheng2025reasoning}, $k \in \{2,4,8,16,32\}$ for both in-domain and out-of-domain evaluations, show in Figure.~\ref{fig:passk} and Table.~\ref{tab:pass@k}. DIVER consistently outperforms baselines across most benchmarks, with particularly notable advantages on AMC, MATH500, OlympiaBench and GPQA. 
We evaluate Pass@k, a metric for model reasoning upper-bounds~\citep{cheng2025reasoning}, for $k \in \{2,4,8,16,32\}$ across all tasks (Figure~\ref{fig:passk}, Table~\ref{tab:pass@k}). DIVER consistently outperforms baselines across nearly all benchmarks.
% Moreover, Pass@32 most closely measures whether the model can effectively explore the boundaries of correct reasoning, where DIVER achieves \textbf{50.0} points on the highly challenging AIME25, surpassing the best Entropy-RL by \textbf{+6.7} points, and reaches \textbf{87.0} points on MMLU-Pro, outperforming GRPO w/ Clip-higher by \textbf{+1.2} points. 
Moreover, Pass@32 most closely reveals reasoning exploration scope, where DIVER achieves superior performance on challenging benchmarks with \textbf{50.0} on AIME25 (\textbf{+6.7} over Entropy-RL) and \textbf{68.9} points on OlympiadBench (\textbf{+2.1} over GRPO w/ Clip-higher). Overall, compared to existing exploration methods, DIVER uniquely enhances Pass@k capability without compromising Pass@1 performance.
% Overall, compared to existing exploration method, DIVER demonstrates the unique capability to enhance Pass@k capability without compromising Pass@1 performance. 
% These comprehensive results establish that DIVER possesses superior capacity for effective exploration and versatile reasoning.
% To further validate these results and examine the qualitative aspects of DIVER's exploration strategy, we conduct case studies that showcase the clear differences between models in multi-attempt problem solving. As detailed in Appendix \ref{appendix:case}, DIVER generates diverse yet coherent reasoning paths that lead to correct solutions. In contrast, Entropy-RL only explores at specific decision points, while GRPO w/ Clip-higher and Pass@k Training show wide but unproductive exploration. All three alternative exploration strategies ultimately lead to incorrect answers. These cases demonstrate that effective exploration requires both sufficient diversity and focus on promising reasoning directions. DIVER achieves this balance by maintaining global exploration coverage while ensuring alignment with correct reasoning paths.
To visually demonstrate our conclusion, we examine case studies of multi-attempt rollouts (Appendix \ref{appendix:case}). 
DIVER generates diverse yet coherent reasoning paths leading to correct solutions, while Entropy-RL explores at specific decision points, and GRPO w/ Clip-higher and Pass@k Training exhibit wide but unproductive exploration, all leading to incorrect answers. These results confirm DIVER's superior exploration scope unlocks higher reasoning capacity.

\textcolor{my_blue}{To address concerns about the scale of Pass@k, we extend our experiments to k=128, 256, 512, and 1024 across AIME24, AIME25, and AMC benchmarks, including the base model Qwen2.5-Math-7B and all baselines. The DIVER demonstrates consistent superiority of DIVER (table~\ref{tab:larger_pass@k}) across all settings: 1) DIVER achieves the highest Pass@k scores across nearly all benchmarks and k values, reaching 86.7\% on AIME24 and 100\% on AMC at Pass@1024. 2) The performance gap increases with larger k, highlighting DIVER's superior ability to generate diverse correct solutions. On AIME24, DIVER outperforms the second-best baseline by 6.7 points at Pass@1024. 3)All RL methods substantially improve over the base model, but DIVER consistently achieves the best results, confirming that diversity-driven exploration provides significant advantages.}

% \paragraph{Diversity and Quality in Multi-attempt Reasoning Rollouts}
% To compare whether DIVER truly enhances diversity in multi-attempt problem solving, we present case study the rollouts generated by the exploration-based baseline models given the same prompt in Appendix .

% \subsection{Analysis}\label{sec:analyses}

\subsection{Analysis}\label{sec:analyses}

\begin{wrapfigure}{r}{0.52\textwidth}
  \centering 
  \vspace{-15pt}
  \includegraphics[width=\linewidth]{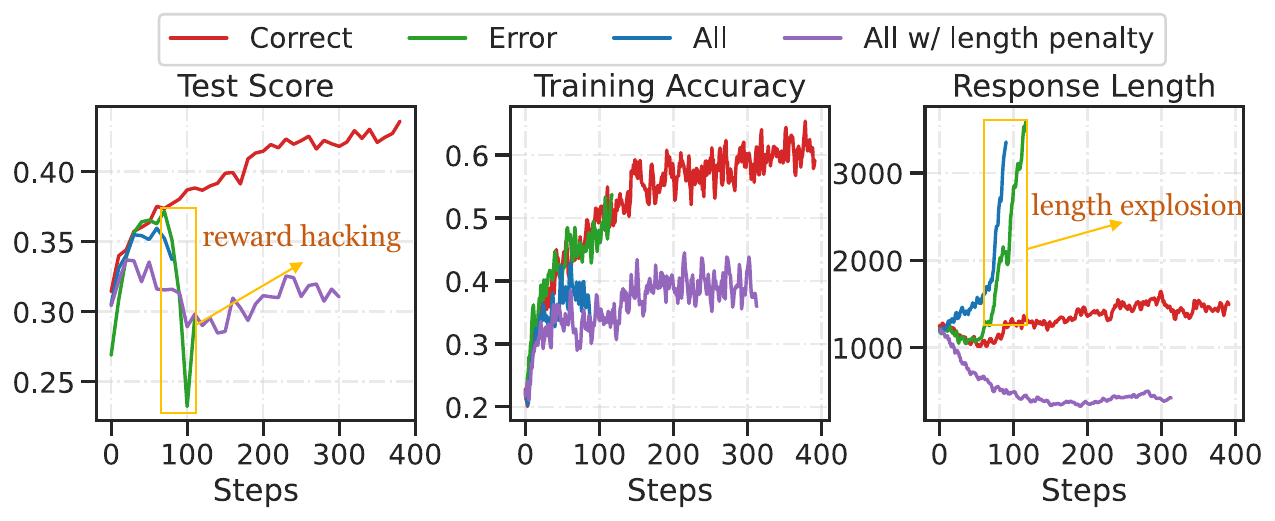}
  % \vspace{-20pt}
  % \caption{Implementation of DIVER comparing different shaping strategies. Left: Average test score on in-domain tasks. Middle: The training accuracy over time. Right: The average training response length. The four lines represent: correct (red), all (green), error (blue), and all samples with length penalty (purple).}
  \caption{DIVER with rewarding diversity in correct, all, error, and all (w/ length penalty) responses.
  % Left: Average test score on in-domain tasks. Middle: Training accuracy over time. Right: Average training response length.
  }
  \vspace{-6pt}
  \label{fig:hacking}
  % \vspace{-10pt}
\end{wrapfigure}
\paragraph{Mitigating Reward Hacking.}

We compare four shaping strategies: applying diversity rewards to (1) correct responses, (2) all responses, (3) errors only, and (4) all responses with a length penalty.
As shown in Figure~\ref{fig:hacking}, rewarding diversity on incorrect or all responses severely degrades test performance, with response lengths exploding as the model exploits long rollouts to gain higher diversity bonuses. Introducing a length penalty mitigates length explosion but still results in poor accuracy. In contrast, the \textit{conditional shaping} strategy, which applies diversity rewards only to correct responses (red line), effectively constrains exploration to valid solutions and mitigates reward hacking.

\begin{wrapfigure}{r}{0.52\textwidth}
  \centering 
  \vspace{-1.5em}
  \includegraphics[width=\linewidth]{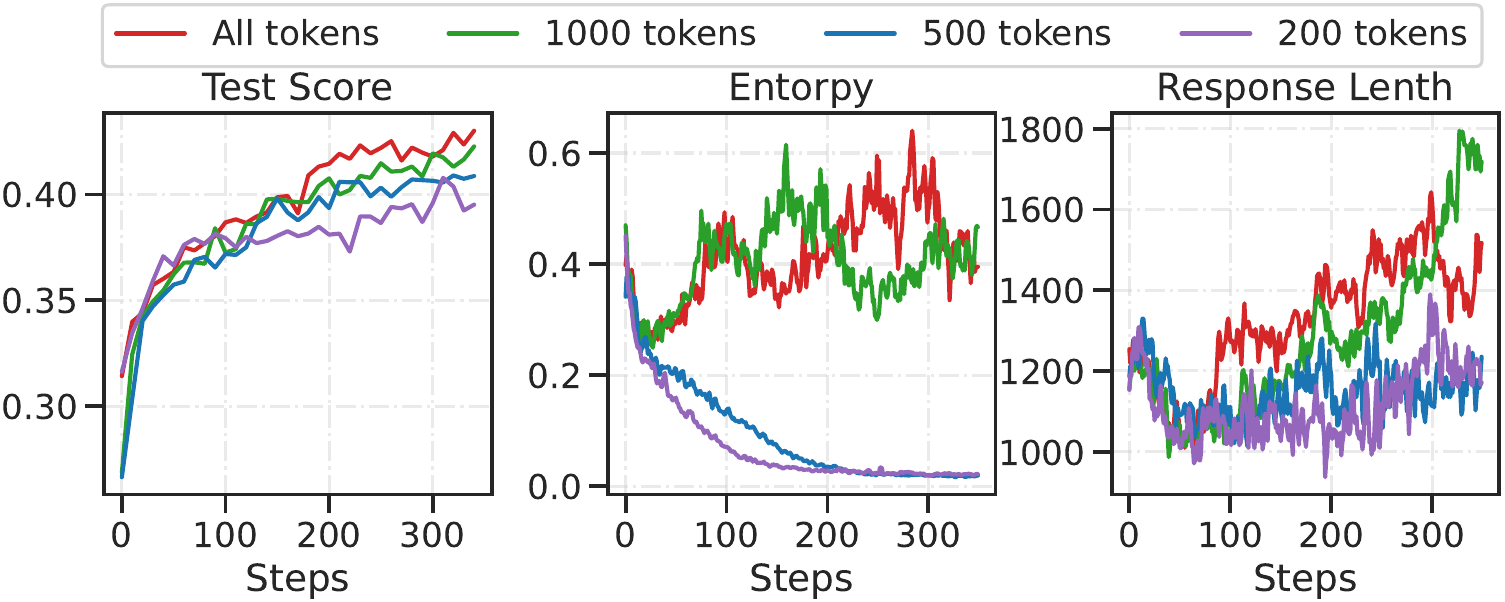}
  % \vspace{-10pt}
  \caption{DIVER with varying horizon constraints. A longer horizon allows for higher global diversity.
  % Left: Average test score on in-domain tasks. Middle: Training entropy loss. Right: Average training response length.
  }
  % \vspace{-10pt}
  \label{fig:token_analyze}
  \vspace{-1em}
\end{wrapfigure}

\paragraph{Longer Horizons Improve Performance.}
To investigate suitable exploration horizons (i.e., range of text for diversity calculation) for reasoning, we evaluate diversity metrics across different token horizons (i.e., the first 200, 500, 1000 tokens of the trajectory) versus complete responses. Figure \ref{fig:token_analyze} shows full responses maximize performance, while shorter horizons significantly reduce both entropy and performance. The consistent improvement with increasing horizon length confirms the reasoning benefits from global sequence-level diversity.
\begin{wrapfigure}{ht}{0.43\textwidth}
  \centering 
  \vspace{-15pt}
  \includegraphics[width=\linewidth]{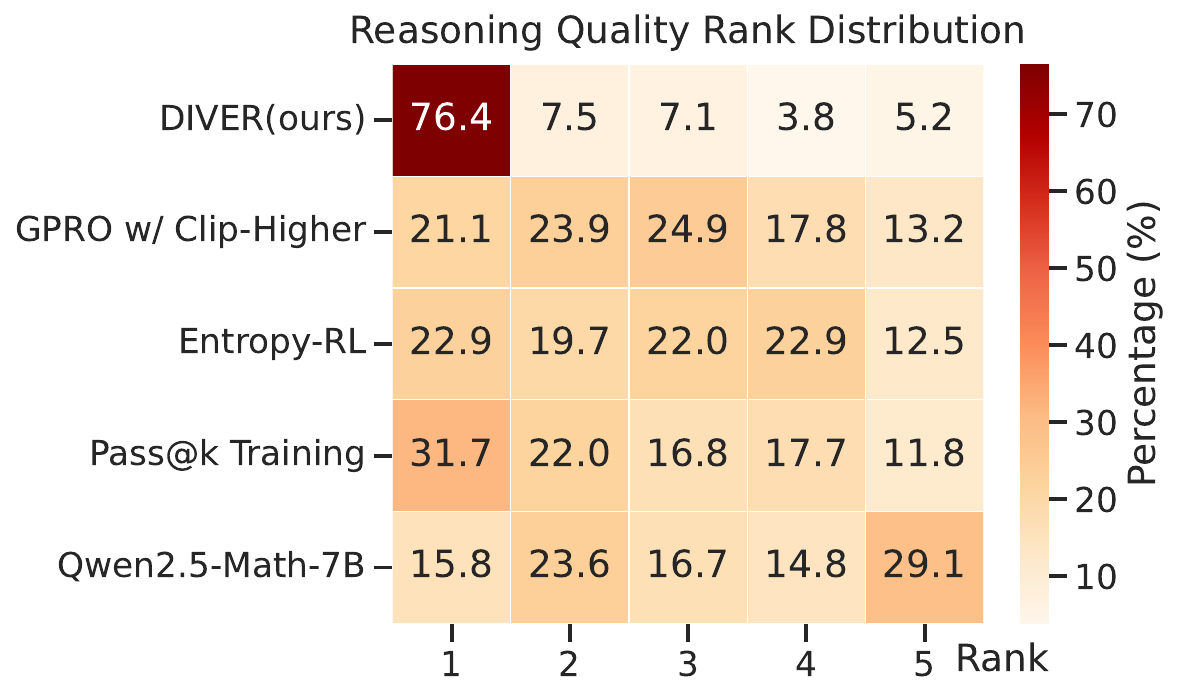}
  % \vspace{-10pt}
  \caption{\textcolor{my_blue}{Average Ranking reasoning quality across six mathematical benchmarks as judged by DeepSeek-V3.2-Exp. Complete results are available in Table~\ref{tab:llm_judge_rank}.
  }}
  \vspace{-10pt}
  \label{fig:llm_judge}
  % \vspace{-1em}
\end{wrapfigure}
\paragraph{Diversity Enhances Reasoning Quality}\textcolor{my_blue}{To validate whether our approach achieves higher quality and more meaningful reasoning results, we evaluate all test responses using \texttt{deepseek-ai/DeepSeek-V3.2-Exp} as a judge model, ranking responses generated by the base model, baselines, and DIVER for each prompt. The evaluation comprehensively assesses four key dimensions: \texttt{Completeness}, \texttt{Logical Consistency}, \texttt{Relevance}, and \texttt{Efficiency}. 
% The table~\ref{tab:llm_judge_rank} presents the final result over all six benchmarks: AIME24/25, AMC, MATH-500, Minerva, and Olympiad. 
% Table.~\ref{tab:llm_judge_rank} demonstrates DIVER's consistent superiority, achieving a 1.52 average ranking compared to baseline methods ranging from 2.57 to 3.19. 
Fig.~\ref{fig:llm_judge} demonstrates DIVER solutions rank best (Rank 1) in 76.4\% of cases, more than double the next best method. This confirms our diversity measures capture meaningful improvements in reasoning quality, not merely surface-level variations.}

\section{Conclusions, Limitations, and Future Work}
% We presented DIVER, a novel approach that leverages global sequence-level diversity to enhance LLM reasoning capabilities. Our empirical study revealed a strong correlation between response diversity and reasoning performance, leading us to optimize for diversity during RL training. By formulating diversity as an intrinsic reward with potential-based shaping, DIVER effectively incentivizes meaningful exploration while preventing reward hacking through conditional shaping. Evaluations show DIVER outperforms previous methods on both in-domain math reasoning benchmarks and demonstrates stronger generalization on out-of-domain tasks.
We presented DIVER, an innovative approach that enhances LLM reasoning from a new perspective of diversity. 
% By formulating diversity as an intrinsic reward with potential-based shaping, DIVER effectively incentivizes meaningful exploration while preserving policy invariance, which we formally prove. 
In contrast to existing methods that focus primarily on local token-level diversity, we examine the role of global sequence-level diversity in incentivizing deep exploration, revealing a positive correlation with reasoning capacity. 
% We formulate global diversity as an intrinsic reward to explicitly optimize the diversity during RL training, and introduce simple yet effective heuristics to mitigate possible reward hacking.
Evaluations showed DIVER achieves consistently higher reasoning capabilities on in-domain tasks and stronger generalization on out-of-domain tasks.
DIVER considers single-turn RLVR, while multi-turn settings hold greater promise for unlocking agent RL's potential in real-world applications.
We leave it as future work.
Another direction is to employ more powerful diversity metrics, such as using LLM-as-a-judge~\citep{gu2024survey}.

\clearpage

\section*{Acknowledgements}
We extend our gratitude to all the reviewers for their valuable feedback and suggestions. This work was completed during an internship at the Shanghai Artificial Intelligence Laboratory.
It was supported in part by the National Natural Science Foundation of China (No. 62376122), and in part by the National Key Research and Development Program of China (No. 2025YFA1016904).

\section*{Ethics Statement}
We are not aware of any major ethical concerns arising from our work. Our study is conducted entirely within the mathematics domain, using only publicly available models and datasets for training and evaluation. No human subjects were involved, and our research does not introduce sensitive or potentially harmful insights.
\section*{Reproducibility Statement}
We provide the experimental setups in Sec.\ref{sec:ex_setup}, with with further details in Appendix~\ref{appendix:ex_details}. The code is available at supplementary material, and will be released on GitHub in the future. Additionally, we will make the weights of the DIVER models publicly available via platforms such as Hugging Face Community upon acceptance.
% \section*{LLMs Usage Statement}
% This work used Large Language Models (LLMs) to assist with polishing the writing and improving text readability. Specifically, LLMs are employed for proofreading, enhancing grammar, and refining sentence structure. The LLM was used solely for editorial purposes to improve clarity and did not contribute to research ideation, experimental design, implementation, analysis, or scientific conclusions. All core research contributions, experiments, and analyses were conducted independently by the authors without LLM assistance.

% \clearpage
\bibliography{iclr2026_conference}
\bibliographystyle{iclr2026_conference}

\clearpage
\appendix 
\part*{Appendix}
\addcontentsline{toc}{part}{Appendix}

\startcontents[appendix]
\printcontents[appendix]{}{1}{}

\newpage

% \section*{Appendix A. Optimal Policy Invariance in DIVER}\label{app:reward_shaping}
\section{Optimal Policy Invariance in DIVER}\label{app:reward_shaping}

% In this section, we give the proof of Theorem~\ref{theo:shaping}, which guarantees the optimal policy invariance when incorporating global diversity as an intrinsic reward.

\textcolor{my_blue}{Following the classical reward shaping study~\citep{ng1999policy}, we give the proof of Theorem~\ref{theo:shaping}, which guarantees the optimal policy invariance when incorporating global diversity as an intrinsic reward.
}

\setcounter{theorem}{0}
\begin{theorem}[Optimal Policy Invariance]
    Let $M\!=\!(S,A,T,R,\gamma)$ denote the MDP for the LLM reasoning task.
    $d(\cdot)\!:S\mapsto \mathbb{R}$ is a real-valued function that computes the sentence-level diversity $d(s)$ of the state $s$ within a group of rollouts.
    We formulate $R_{\text{int}}(\cdot)\!:S\!\times\! A\!\times\! S\mapsto \mathbb{R}$ as an intrinsic reward function that is the difference between sentence diversities of two adjacent states, such that for all $s\in S, a\in A, s'\in S$, $R_{\text{int}}(s,a,s')=\gamma d(s')-d(s)$. 
    Then, with any constant balancing ratio $\lambda$, every optimal policy in the transformed MDP $M'\!=\!(S,A,T,R+\lambda R_{\text{int}},\gamma)$ will also be an optimal policy in $M$, and vice versa.
\end{theorem}

\begin{proof}
    % If $\gamma=1$, then since replacing $\Phi(s)$ with $\Phi'(s)-k$ for any constant $k$ would not change the shaping rewards $R_{\text{int}}$ (which is difference of these potentials), we may, by replacing $\Phi(s)$ with $\Phi(s)-\Phi(s_0)$ if necessary, assume without loss of generality that the $\Phi$ used to express $R_{\text{int}}$ satisfies $\Phi(s_0)=0$.
    
    For the original MDP $M$, we know that its optimal Q-function $Q_M^*$ satisfies the Bellman optimality equation~\citep{sutton2018reinforcement}:
    \begin{equation}
        Q_M^*(s,a)=\mathbb{E}_{s'}\left[R(s,a,s')+\gamma\max_{a'\in A}Q_M^*(s',a')\right].
    \end{equation}
    With some simple algebraic manipulation, we can get:
    \begin{equation}
        Q_M^*(s,a) - \lambda d(s) = \mathbb{E}_{s'}\left[R(s,a,s')+\lambda\Bigl(\gamma d(s')-d(s)\Bigr)+\gamma\max_{a'\in A}\Bigl(Q_M^*(s',a')-\lambda d(s')\Bigr)\right].
    \end{equation}

    If we now define $\hat{Q}_{M'}(s,a)\triangleq Q_M^*(s,a)-\lambda d(s)$ and substitute that and $R_{\text{int}}(s,a,s')=\gamma d(s')-d(s)$ into the previous equation, we can get:
    \begin{equation}
    \begin{aligned}
        \hat{Q}_{M'}(s,a) & = \mathbb{E}_{s'}\left[R(s,a,s')+\lambda R_{\text{int}}(s,a,s')+\gamma\max_{a'\in A}\hat{Q}_{M'}(s',a')\right] \\
        & = \mathbb{E}_{s'}\left[R'(s,a,s')+\gamma\max_{a'\in A}\hat{Q}_{M'}(s',a')\right],
    \end{aligned}
    \end{equation}
    which is exactly the Bellman optimality equation for the transformed MDP $M'$, where $R'=R+\lambda R_{\text{int}}$ is the reward function for $M'$.
    % For the undiscounted case, we moreover have $\hat{Q}_{M'}(s_0,a)=Q_M^*(s_0,a)-\Phi(s_0)=0-0=0$.
    % So, $\hat{Q}_{M'}(s,a)$ satisfies the Bellman optimality equations for $M'$, and must in fact be the unique optimal Q-function. 
    Thus, $Q_{M'}^*(s,a)=\hat{Q}_{M'}(s,a)=Q_M^*(s,a)-\lambda d(s)$, and the optimal policy for $M'$ therefore satisfies:
    \begin{equation}
    \begin{aligned}
        \pi_{M'}^*(s) & = \arg\max_{a\in A}Q_{M'}^*(s,a) \\
        & = \arg\max_{a\in A}\Bigl[Q_M^*(s,a) - \lambda d(s)\Bigr] \\
        & = \arg\max_{a\in A}Q_M^*(s,a),
    \end{aligned}
    \end{equation}
    and is therefore also optimal in $M$.
    To show every optimal policy in $M$ is also optimal in $M'$, simply apply the same proof with the roles of $M$ and $M'$ interchanged (and using $-R_{\text{int}}$ as the intrinsic reward).
    This completes the proof.
\end{proof}

\clearpage

% \section*{Appendix B. Experimental Details}
\section{Experimental Details}

\subsection{Diversity based Filtering}
\label{app:filter}
To compare the model training process using rollouts with different diversity levels, we conduct an illustrative experiment to filter GRPO rollouts into high-diversity and low-diversity subsets to train separate models. Specifically, the policy generates $2*G$ responses for each query. For high-diversity scheme, $G$ responses of highest diversity metrics (TD and ED) are filtered as the candidate responses ${o_1, ..., o_G}$ for GRPO training. For low-diversity scheme, responses of lowest diversity metrics are filtered. Both approaches maintain identical computational costs, as they use the same number of generated responses and differ only in the selection criteria applied during filtering.

% \subsection*{B.1 Hyperparameter settings}
\subsection{Detailed settings}
\label{appendix:ex_details}

\paragraph{Datasets} Our training data is a subset of OpenR1-Math-220k~\citep{openr1}, with prompts collected from NuminaMath 1.5~\citep{numina_math_datasets}. We follow the LUFFY~\citep{yan2025learning}\footnote{\url{https://huggingface.co/datasets/Elliott/Openr1-Math-46k-8192}} dataset construction methodology but differ in that we do not incorporate off-policy reasoning traces, as ours is a purely on-policy approach.

\paragraph{Evaluation} We evaluate our method on six mathematical reasoning benchmarks: AIME 2024\footnote{\url{https://huggingface.co/datasets/HuggingFaceH4/aime_2024}}, AIME 2025\footnote{\url{https://huggingface.co/datasets/PrimeIntellect/AIME-25}}, AMC~\citep{li2024numinamath}, Minerva~\citep{dataset_minerva}, OlympiadBench~\citep{dataset_olympiad}, and MATH-500~\citep{dataset_math}. Our main results report avg@32 for the smaller test sets (AIME 2024, AIME 2025, AMC) and pass@1 for the larger benchmarks. For Pass@k evaluation, we generate $k$ completions and select the one with the highest reward score.
For cross-domain generalization, we test on ARC-c~\citep{arc}, GPQA-diamond (GPQA$^*$)~\citep{gpqa}, and MMLU-Pro.

\paragraph{RL Practice} 
We set $\beta=0$ to remove the KL loss term and use 0.28 for higher clip following GPPO w/ Clip-higher. Detailed implementation parameters are provided in Table~\ref{tab:hyper}. All training experiments are conducted using 8 A100 GPUs. We train 350 steps Qwen2.5-Math-7B, and 200 steps for others. Our implementation is based on verl\footnote{https://github.com/volcengine/verl}, which uses vLLM\footnote{https://github.com/vllm-project/vllm} as the rollout generators. We are thankful for these open-source repositories. 

\begin{table}[h]
\centering
\caption{\textcolor{my_blue}{Computation overhead analysis.}}
\label{tab:comp_cost}
\setlength{\tabcolsep}{2.5pt}  
\renewcommand{\arraystretch}{1.3} 
\resizebox{0.85\textwidth}{!}{%
\begin{tabular}{lcccc}
\toprule
& \textbf{DIVER(ours)} & \textbf{GRPO w/ Clip-Higher} & \textbf{Entropy-RL} & \textbf{Pass@k Training} \\
\midrule
Time Cost (350 Steps) & 29.75h & 28.46h & 28.13h & 30.73h \\
\bottomrule
\end{tabular}
}
\end{table}

\textcolor{my_blue}{The computation overhead of calculating textual diversity and equational diversity is negligible, since both metrics only involve rule-based calculations (e.g., n-gram matching for BLEU score or string recognition for equation extraction) without any feedforward or backpropagation of large-scale models. As shown in Table~\ref{tab:comp_cost}, DIVER only incurs a 5\% increase in training time compared to GRPO and Entropy-RL baselines, while requiring less training time than Pass@k Training baseline.}
\begin{table}[ht]
\centering
\caption{Hyperparameter settings}
\label{tab:hyper}
\begin{tabular}{lrlr}
\hline
\toprule[0.5pt]
\multicolumn{1}{l}{\bf Hyperparameter}  &\multicolumn{1}{l}{\bf Value} &
\multicolumn{1}{l}{\bf Hyperparameter}  &\multicolumn{1}{l}{\bf Value} \\
\hline
max prompt length & 1024        & KL coefficient $\beta$ & 0.0\\
max response length & 8192      & train temperature & 1.0\\
num generations $G$ & 8         & eval temperature & 0.6\\
gpu memory utilization & 0.85   & entropy coefficient &0.0\\
learning rate & 1e-6            & high clip ratio $\epsilon_h$ & 0.28 \\
train batch size & 128          & low clip ratio $\epsilon_l$ & 0.20\\
mini batch size & 32            & shaping ratio $\lambda$ & 0.1 \\
use dynamic batch size & True   &  diversity upper bound $\sigma$ & 0.65 \\
validate batch size & 512 \\
\hline
\toprule[0.5pt]
\end{tabular}
\end{table}

\subsection{System Prompt}
We use the same system prompt for training and inference in all our models except LLaMA-3.1-8B-Instruct:
\begin{tcolorbox}[
  colback=lightgray!30, 
  colframe=black, 
  width=\textwidth, 
  arc=0mm, 
  boxrule=0.5pt,  
]

Your task is to follow a systematic, thorough reasoning process before providing the final solution. This involves analyzing, summarizing, exploring, reassessing, and refining your thought process through multiple iterations. Structure your response into two sections: Thought and Solution. In the Thought section, present your reasoning using the format: ``\texttt{<think>}\textbackslash n thoughts \texttt{</think>}\textbackslash n''. Each thought should include detailed analysis, brainstorming, verification, and refinement of ideas. After ``\texttt{</think>}\textbackslash n'' in the Solution section, provide the final, logical, and accurate answer, clearly derived from the exploration in the Thought section. If applicable, include the answer in \texttt{\textbackslash boxed\{\}} for closed-form results like multiple choices or mathematical solutions.

\textbf{User:} This is the problem: \{QUESTION\}

\textbf{Assistant:} \texttt{<think>}
\end{tcolorbox}
For LLaMA-3.1-8B-Instruct, we use a simplified prompt which only includes the CoT prompt:
\begin{tcolorbox}[
  colback=lightgray!30, 
  colframe=black, 
  width=\textwidth, 
  arc=0mm, 
  boxrule=0.5pt,  
]

\textbf{User:}  \{QUESTION\}

\textbf{Answer:} Let's think step by step. 
\end{tcolorbox}

\clearpage

% \section*{Appendix C. Detailed Matrics Definition}
\section{Detailed Metrics Definition}

\subsection{Pass@k Performance}
Given a question $x$, we employ the model to generate $k$ independent and identically distributed (i.i.d.) responses. Each response is evaluated by a binary reward function, yielding $r_i \in {0, 1}$ where $r_i = 1$ indicates a correct response. The pass@$k$ metric quantifies the probability of obtaining at least one correct response among the $k$ samples:

\begin{equation}
    \text{pass@k} = \mathbb{P}\left[\bigvee_{i=1}^{k} (r_i = 1)\right] \\
    = \mathbb{E}\left[1 - \prod_{i=1}^{k}(1 - r_i)\right]
\end{equation}

While pass@1 evaluates evaluates the model's accuracy on first attempts, pass@k metric emphasizes the model's ability to generate diverse solutions and improve success rates through sampling. \textcolor{my_blue}{To rigorously demonstrate DIVER's consistent advantage, we repeat all experiments across Pass@2 to Pass@32, and report the means and standard deviations for a total of 3 runs.} As shown in figure \ref{fig:passk} and table \ref{tab:pass@k}, our method consistently outperforms all baselines approaches (GRPO w/ Clip-higher, Entropy-RL, and Pass@k Training) across the spectrum of pass@k metrics (k=2 to k=32) on both in-distribution benchmarks and out-of-distribution benchmarks. The performance gap is particularly significant at higher k values, where DIVER demonstrates superior exploration capabilities and achieves the highest average scores.

\begin{table*}[ht]
\centering

\caption{\textcolor{my_blue}{pass@k performance.}}
\label{tab:pass@k}
\setlength{\tabcolsep}{2.5pt}  
\renewcommand{\arraystretch}{1.3} 
\resizebox{\textwidth}{!}{%
\begin{tabular}{c|c|ccccc>{}c|ccc>{}c}
\toprule
\multirow{2}{*}{\textbf{Pass@k}}&\multirow{2}{*}{\textbf{Method}} & \multicolumn{6}{c}{\textbf{In-Distribution Performance}} & \multicolumn{4}{c}{\textbf{Out-of-Distribution Performance}} \\
\cmidrule(lr){3-8} \cmidrule(lr){9-12}
 & &\textbf{AIME 24/25} & \textbf{AMC} & \textbf{MATH-500} & \textbf{Minerva} & \textbf{Olympiad} & \textbf{Avg.} & \textbf{ARC-c} & \textbf{GPQA}$^{*}$ & \textbf{MMLU-Pro} & \textbf{Avg.} \\
\midrule

\multirow{5}{*}{\makecell{$k=2$}}
&Qwen2.5-Math-7B
& 23.3 $\pm$ 3.3/12.2 $\pm$ 5.1 & 55.4 $\pm$ 1.2 & 74.7 $\pm$ 0.8 & 27.2 $\pm$ 1.3  & 35.5 $\pm$ 0.8 & 38.1 $\pm$ 0.3 & 89.2 & 52.0 & 57.4 & 66.2 \\
&GPRO w/ Clip-higher
& 26.6 $\pm$ 5.8/18.9 $\pm$ 1.9 & 61.4 $\pm$ 3.7 & 85.7 $\pm$ 1.5 & 38.0 $\pm$ 0.2  & 50.4 $\pm$ 0.7 & 46.8 $\pm$ 1.3 & 89.2 & 52.0 & 57.4 & 66.2 \\
&Entropy-RL
& 26.7 $\pm$ 0.0/22.2 $\pm$ 3.9 & 63.9 $\pm$ 1.2 & 86.3 $\pm$ 0.5 & 35.5 $\pm$ 0.8  & 50.5 $\pm$ 2.4 & 47.5 $\pm$ 0.7 & 88.1 & 55.1 & 58.3 & 67.2 \\
&Pass@k Training
& 23.3 $\pm$ 3.4/22.2 $\pm$ 6.9 & 63.5 $\pm$ 0.7 & 85.1 $\pm$ 1.0 & \textbf{40.7 $\pm$ 1.2}  & 47.6 $\pm$ 1.6 & 47.1 $\pm$ 1.7 & 89.2 & 53.6 & \textbf{60.5} & 67.8 \\
\cmidrule(lr){2-8} \cmidrule(lr){9-12}
&\textbf{DIVER(ours)}
& \textbf{28.9 $\pm$ 5.1}/\textbf{27.8 $\pm$ 8.4} & \textbf{67.9 $\pm$ 1.4} & \textbf{87.7 $\pm$ 0.1} & 42.9 $\pm$ 1.1  & \textbf{52.1 $\pm$ 1.5} & \textbf{51.2 $\pm$ 1.9} & \textbf{90.3} & \textbf{57.7} & 60.2 & \textbf{69.4} \\
\midrule

\multirow{5}{*}{\makecell{$k=4$}}
&Qwen2.5-Math-7B
& 30.0 $\pm$ 5.8/16.6 $\pm$ 5.8 & 67.1 $\pm$ 3.5 & 81.5 $\pm$ 0.6 & 35.5 $\pm$ 0.6  & 43.8 $\pm$ 0.5 & 45.8 $\pm$ 1.7 & 92.3 & 65.0 & 66.4 & 74.6 \\
&GPRO w/ Clip-higher
& 32.2 $\pm$ 11.7/25.5 $\pm$ 3.9 & 72.3 $\pm$ 1.2 & 88.9 $\pm$ 1.1 & \textbf{47.3 $\pm$ 1.1}  & 53.3 $\pm$ 0.9 & 53.2 $\pm$ 0.9 & 92.3 & 65.0 & 66.4 & 74.6 \\
&Entropy-RL
& \textbf{38.9 $\pm$ 3.8}/\textbf{28.9 $\pm$ 3.8} & 74.3 $\pm$ 2.8 & 89.5 $\pm$ 0.1 & 40.6 $\pm$ 1.2  & 54.7 $\pm$ 0.8 & 54.5 $\pm$ 1.3 & 92.2 & 64.3 & 66.6 & 74.4 \\
&Pass@k Training
& 36.7 $\pm$ 3.4/24.4 $\pm$ 8.4 & 71.5 $\pm$ 3.0 & 88.5 $\pm$ 0.2 & 42.2 $\pm$ 1.1  & 55.7 $\pm$ 0.7 & 53.2 $\pm$ 1.6 & 92.2 & 66.5 & 64.8 & 74.5 \\
\cmidrule(lr){2-8} \cmidrule(lr){9-12}
&\textbf{DIVER(ours)}
& 37.8 $\pm$ 1.9/24.4 $\pm$ 3.9 & \textbf{75.9 $\pm$ 4.3} & \textbf{90.5 $\pm$ 0.5} & \textbf{47.3 $\pm$ 1.9}  & \textbf{57.7 $\pm$ 0.4} & \textbf{55.6 $\pm$ 1.6} & \textbf{93.3} & \textbf{70.4} & \textbf{69.1} & \textbf{77.6} \\

\midrule

\multirow{5}{*}{\makecell{$k=8$}}
&Qwen2.5-Math-7B
& 42.2 $\pm$ 5.1/23.3 $\pm$ 3.4 & 77.1 $\pm$ 2.1 & 86.7 $\pm$ 1.2 & 41.3 $\pm$ 0.6  & 51.3 $\pm$ 0.9 & 53.7 $\pm$ 1.1 & 95.3 & 75.5 & 74.3 & 81.7 \\
&GPRO w/ Clip-higher
& 35.6 $\pm$ 5.1/30.0 $\pm$ 5.8 & 77.9 $\pm$ 1.8 & 90.7 $\pm$ 0.3 & \textbf{52.1 $\pm$ 1.3}  & 58.5 $\pm$ 2.1 & 57.4 $\pm$ 1.1 & 95.3 & 75.5 & 74.3 & 81.7 \\
&Entropy-RL
& 41.1 $\pm$ 1.9/28.9 $\pm$ 3.8 & 80.3 $\pm$ 2.5 & 91.9 $\pm$ 0.8 & 45.6 $\pm$ 1.0  & 59.9 $\pm$ 0.5 & 58.0 $\pm$ 1.2 & 94.1 & 71.4 & 74.2 & 79.9 \\
&Pass@k Training
& 36.7 $\pm$ 3.4/24.4 $\pm$ 2.0 & 80.3 $\pm$ 0.7 & 90.7 $\pm$ 0.1 & 45.5 $\pm$ 1.9  & 60.2 $\pm$ 0.4 & 56.3 $\pm$ 0.6 & 93.3 & 77.0 & 71.6 & 80.6 \\
\cmidrule(lr){2-8} \cmidrule(lr){9-12}
&\textbf{DIVER(ours)}
& \textbf{45.6 $\pm$ 1.9}/\textbf{30.0 $\pm$ 0.0} & \textbf{83.9 $\pm$ 2.5} & \textbf{92.9 $\pm$ 0.1} & 50.6 $\pm$ 1.1  & \textbf{61.8 $\pm$ 0.7} & \textbf{60.8 $\pm$ 0.0} & \textbf{95.9} & \textbf{81.1} & \textbf{76.3} & \textbf{84.4} \\
\midrule

\multirow{5}{*}{\makecell{$k=16$}}
&Qwen2.5-Math-7B
& 43.3 $\pm$ 12.0/23.3 $\pm$ 3.4 & 81.1 $\pm$ 3.0 & 89.9 $\pm$ 0.3 & 49.1 $\pm$ 0.8  & 56.5 $\pm$ 1.5 & 56.0 $\pm$ 2.2 & 96.8 & 84.7 & 80.2 & 87.2 \\
&GPRO w/ Clip-higher
& 45.6 $\pm$ 5.1/37.8 $\pm$ 5.1 & 84.3 $\pm$ 2.4 & 92.3 $\pm$ 0.3 & \textbf{59.3 $\pm$ 2.2}  & 63.2 $\pm$ 0.5 & 63.7 $\pm$ 2.3 & 96.8 & 84.7 & 80.2 & 87.2 \\
&Entropy-RL
& 47.8 $\pm$ 5.1/35.6 $\pm$ 2.0 & \textbf{88.0 $\pm$ 2.1} & 93.7 $\pm$ 0.2 & 50.7 $\pm$ 1.6  & 62.8 $\pm$ 0.3 & 63.2 $\pm$ 0.6 & 95.7 & 78.6 & 79.7 & 84.7 \\
&Pass@k Training
& \textbf{50.0 $\pm$ 6.7}/34.4 $\pm$ 8.4 & 82.7 $\pm$ 2.5 & 92.3 $\pm$ 0.1 & 50.6 $\pm$ 1.8  & 64.0 $\pm$ 0.5 & 62.3 $\pm$ 0.2 & 95.9 & 82.7 & 77.7 & 85.4 \\
\cmidrule(lr){2-8} \cmidrule(lr){9-12}
&\textbf{DIVER(ours)}
& 47.8 $\pm$ 5.1/\textbf{38.9 $\pm$ 3.8} & \textbf{88.0 $\pm$ 0.0} & \textbf{94.4 $\pm$ 0.4} & 56.5 $\pm$ 2.9  & \textbf{66.0 $\pm$ 0.7} & \textbf{65.3 $\pm$ 1.0} & \textbf{97.3} & \textbf{89.3} & \textbf{81.7} & \textbf{89.4} \\
\midrule

\multirow{5}{*}{\makecell{$k=32$}}
&Qwen2.5-Math-7B
& 52.2 $\pm$ 5.1/30.0 $\pm$ 3.3 & 88.8 $\pm$ 3.5 & 92.1 $\pm$ 0.5 & 54.4 $\pm$ 1.0  & 62.2 $\pm$ 0.8 & 63.3 $\pm$ 0.9 & 97.9 & 89.8 & 85.8 & 91.2 \\
&GPRO w/ Clip-higher
& 56.7 $\pm$ 10.0/42.2 $\pm$ 3.9 & 92.0 $\pm$ 2.8 & 94.4 $\pm$ 0.4 & 54.9 $\pm$ 0.4  & 65.8 $\pm$ 0.6 & 67.6 $\pm$ 0.8 & 97.9 & 89.8 & 85.8 & 91.2 \\ 
&Entropy-RL
& 54.4 $\pm$ 2.0/42.2 $\pm$ 1.9 & 90.9 $\pm$ 3.0 & 94.1 $\pm$ 0.3 & 51.4 $\pm$ 0.3  & 66.3 $\pm$ 0.5 & 67.7 $\pm$ 0.6 & 96.5 & 88.3 & 84.6 & 89.8 \\ 
&Pass@k Training
& 56.7 $\pm$ 5.8/33.3 $\pm$ 6.7 & 88.4 $\pm$ 0.7 & 93.8 $\pm$ 0.3 & 53.8 $\pm$ 1.5  & 68.0 $\pm$ 0.5 & 65.7 $\pm$ 1.7 & 96.9 & 89.2 & 84.0 & 90.0 \\
\cmidrule(lr){2-8} \cmidrule(lr){9-12}
&\textbf{DIVER(ours)}
& \textbf{58.9 $\pm$ 3.8}/\textbf{48.9 $\pm$ 1.9} & \textbf{92.4 $\pm$ 1.9} & \textbf{95.4 $\pm$ 0.2} & \textbf{64.4 $\pm$ 3.9}  & \textbf{69.0 $\pm$ 0.3} & \textbf{70.5 $\pm$ 0.1} & \textbf{98.0} & \textbf{89.8} & \textbf{87.0} & \textbf{91.6} \\ 
\bottomrule
\end{tabular}
}
\end{table*}

\begin{table*}[t]

\centering
\caption{\textcolor{my_blue}{Performance comparison with larger pass@k values.}}
\label{tab:larger_pass@k}
\resizebox{0.8\textwidth}{!}{%
\begin{tabular}{l|l|cccc}
\toprule
\textbf{Benchmark} & \textbf{Method} & \textbf{pass@128} & \textbf{pass@256} & \textbf{pass@512} & \textbf{pass@1024} \\
\midrule
\multirow{5}{*}{AIME24} 
& Qwen2.5-Math-7B & 63.3 & 70.0 & 76.7 & 80.0 \\
& GPRO w/ Clip-Higher & 70.0 & 73.3 & 76.7 & 76.7 \\
& Entropy-RL & 70.0 & 73.3 & 76.7 & 76.7 \\
& Pass@k Training & 70.0 & 76.7 & 80.0 & 80.0 \\
\cmidrule{2-6}
& \textbf{DVIER(ours)} & \textbf{76.7} & \textbf{80.0} & \textbf{83.3} & \textbf{86.7} \\
\midrule
\multirow{5}{*}{AIME25}
& Qwen2.5-Math-7B & 40.0 & 56.7 & 63.3 & 66.7 \\
& GPRO w/ Clip-Higher & 60.0 & 63.3 & 66.7 & 70.0 \\
& Entropy-RL & \textbf{63.3} & 63.3 & 66.7 & 73.3 \\
& Pass@k Training & \textbf{63.3} & 63.3 & \textbf{70.0} & 73.3 \\
\cmidrule{2-6}
& \textbf{DVIER(ours)} & \textbf{63.3} & \textbf{66.7} & \textbf{70.0} & \textbf{76.7} \\
\midrule
\multirow{5}{*}{AMC}
& Qwen2.5-Math-7B & 92.8 & \textbf{97.6} & 97.6 & 97.6 \\
& GPRO w/ Clip-Higher & 94.0 & 96.4 & 97.6 & 98.8 \\
& Entropy-RL & \textbf{96.4} & 96.4 & 97.6 & 98.8 \\
& Pass@k Training & 95.2 & 96.4 & 97.6 & \textbf{100.0} \\
\cmidrule{2-6}
& \textbf{DVIER(ours)} & \textbf{96.4} & \textbf{97.6} & \textbf{98.8} & \textbf{100.0} \\
\bottomrule
\end{tabular}
}
\end{table*}

\subsection{BLEU Score}
\label{app:bleu}

BLUE measures the similarity between a candidate string $c$ and a reference string $r$ by calculating the $n$-gram (short phrases of length $n$) precision, while also penalizing overly short outputs through a brevity penalty as
\begin{equation}
\mathrm{BLEU}=\mathrm{BP}\cdot \exp\!\left(\sum_{n=1}^{N} w_n \log p_n\right),\quad
p_n=\frac{\sum_{g\in G_n}\min\{C_c(g),C_r(g)\}}{\sum_{g\in G_n}C_c(g)},
\end{equation}
% where $p_n$ is the modified $n$-gram precision, $w_n$ are weights ($\sum_n w_n=1$, default $w_n=1/N$), $G_n$ is the set of $n$-grams, and $C_x(g)$ counts how often $g$ appears in $x$. 
where $N$ is the maximum $n$-gram length considered, $p_n$ is the modified precision for $n$-grams of size $n$, $w_n$ is the weight for each $n$-gram level (usually uniform, e.g., $w_n=1/N$), $G_n$ is the set of $n$-grams, and $C_c(g)/C_r(g)$ counts how often the $n$-gram $g$ appears in string $c/r$. 
The brevity penalty $\mathrm{BP}$ is defined as
\begin{equation}
\mathrm{BP}=
\begin{cases}
1, & |c|\ge |r|,\\
\exp(1-|r|/|c|), & \text{otherwise}.
\end{cases}
\end{equation}

\subsection{Equation Extraction}
We extract mathematical formulas using regular expressions that identify three common LaTeX notation patterns: \verb|\\[ \\]|, \verb|\\( \\)|, and \verb|$ $|. Our implementation employs \texttt{re.findall()} to capture these patterns and stores them in a set structure to eliminate duplicates.

\textcolor{my_blue}{To evaluate the extraction accuracy, we compare the automatically extracted equation counts against human-annotated ground-truth values. For demonstration, we randomly sample 5 responses from different query inputs. Table \ref{tab:equation_extraction} presents the equation counts obtained through our extraction method alongside the corresponding ground-truth values. The observed failure rate is approximately $8.4\% (8/95)$.}

\textcolor{my_blue}{We note that since Equational Diversity (ED) serves as an intrinsic reward signal to encourage exploration during training, perfect extraction accuracy is not required. The metric only needs to approximate response diversity with sufficient fidelity to guide effective exploration. The observed accuracy level is adequate for this purpose, as evidenced by the consistent performance improvements shown in our main results.}

\begin{table}[ht]
\centering
\caption{\textcolor{my_blue}{Equation extraction accuracy evaluation on randomly sampled responses.}}
\label{tab:equation_extraction}
\setlength{\tabcolsep}{2.5pt}  
\renewcommand{\arraystretch}{1.3} 
\resizebox{0.85\textwidth}{!}{%
\begin{tabular}{lccccc}
\toprule
& \textbf{Response 1} & \textbf{Response 2} & \textbf{Response 3} & \textbf{Response 4} & \textbf{Response 5} \\
\midrule
No. of equations via extraction & 14 & 25 & 13 & 17 & 18 \\
Ground-truth value & 17 & 25 & 18 & 17 & 18 \\
\bottomrule
\end{tabular}
}
\end{table}

\clearpage
\section{More Analyze}

\subsection{Analysis of Diversity Metrics Validity}
\textcolor{my_blue}{To verify whether TD and ED truly capture semantic-level reasoning diversity, we conduct a correlation analysis comparing them with two semantic similarity metrics: 1) LLM-based scores: We compute the cosine similarity of the model's hidden states from the final layer outputs using \texttt{Qwen2.5-7B-Instruct} for all baselines and our method; 2) Embedding-based scores: We use \texttt{google/embeddinggemma-300m} to encode responses into embeddings and compute their cosine similarity.}

\textcolor{my_blue}{For each method (DIVER and baselines), we generate multiple responses per prompt and compute all diversity metrics within each response group. Evaluated on 1,560 rollouts total, the average metrics are shown in the table~\ref{tab:metric_validity} (↑ indicates higher values represent greater diversity; ↓ indicates the opposite). The results demonstrate strong consistency between TD/ED and LLM-/Embedding-based alternatives. Methods achieving higher TD/ED scores consistently exhibit lower LLM-/Embedding-based semantic similarity scores, confirming that our choice of diversity metrics effectively captures meaningful semantic diversity rather than merely surface-level variation. This validates the generality of TD and ED as reliable proxies for reasoning diversity.}
\begin{table*}[ht]
\centering
\caption{\textcolor{my_blue}{Correlation between diversity metrics across different methods. Higher TD/ED ($\uparrow$) and lower LLM/Embedding similarity scores ($\downarrow$) indicate greater diversity.}}
\label{tab:metric_validity}
\resizebox{\textwidth}{!}{%
\begin{tabular}{l|cc|cc|cc}
\toprule
\textbf{Method} & \textbf{TD ($\uparrow$)} & \textbf{ED ($\uparrow$)} & \textbf{LLM-based ($\downarrow$)} & \textbf{Embedding-based ($\downarrow$)} & \textbf{pass@1} & \textbf{pass@8} \\
\midrule
\textbf{DIVER(ours)} & 0.702 & 0.477 & 0.936 & 0.881 & 43.1 & 60.8 \\
GRPO w/ Clip-Higher & 0.805 & 0.575 & 0.916 & 0.841 & 40.7 & 57.4 \\
Entropy-RL & 0.549 & 0.294 & 0.947 & 0.954 & 41.8 & 58.0 \\
Pass@k Training & 0.441 & 0.200 & 0.959 & 0.967 & 41.5 & 56.3 \\
\bottomrule
\end{tabular}%
}
\end{table*}

\textcolor{my_blue}{Moreover, our diversity metrics are both computationally lightweight and easy to implement. The computation overhead of calculating textual diversity and equational diversity is negligible, since both metrics only involve rule-based calculations (e.g., n-gram matching for BLEU score or string recognition for equation extraction) without any feedforward or backpropagation of large-scale models (e.g., the above LLM-/Embedding-based metrics).}

\subsection{Analysis of Reasoning Quality}
\textcolor{my_blue}{To validate whether our approach achieves higher quality and more meaningful reasoning results, we evaluate all test responses using \texttt{deepseek-ai/DeepSeek-V3.2-Exp} as a judge model, ranking responses generated by the base model, baselines, and DIVER for each prompt. The evaluation comprehensively assesses four key dimensions: \texttt{Completeness}, \texttt{Logical Consistency}, \texttt{Relevance}, and \texttt{Efficiency}. The table~\ref{tab:llm_judge_rank} presents the final result over all six benchmarks: AIME24/25, AMC, MATH-500, Minerva, and Olympiad. It demonstrates DIVER's consistent superiority, achieving a 1.52 average ranking compared to baseline methods ranging from 2.57 to 3.19. DIVER solutions rank best (Rank 1) in 76.4\% of cases—more than double the next best method. This confirms our diversity measures capture meaningful improvements in reasoning quality, not merely surface-level variations.}

\begin{table}[ht]
\centering
\caption{\textcolor{my_blue}{Ranking reasoning quality across six mathematical benchmarks as judged by DeepSeek-V3.2-Exp. Overall Avg. shows aggregate performance and lower Avg. Rank($\downarrow$) is better.}}
\label{tab:llm_judge_rank}
\resizebox{\textwidth}{!}{%
\begin{tabular}{l|l|c|c|c|c|c|c}
\toprule
\textbf{Benchmark} & \textbf{Model} & \textbf{Avg. Rank(↓)} & \textbf{Rank 1(best)} & \textbf{Rank2} & \textbf{Rank 3} & \textbf{Rank4} & \textbf{Rank5 (worst)} \\
\midrule
\multirow{5}{*}{Overall Avg.} 
& Qwen2.5-Math-7B & 3.19 & 15.8\% & 23.6\% & 16.7\% & 14.8\% & 29.1\% \\
& GPRO w/ Clip-Higher & 2.77 & 21.1\% & 23.9\% & 24.9\% & 17.8\% & 13.2\% \\
& Entropy-RL & 2.86 & 22.9\% & 19.7\% & 22.0\% & 22.9\% & 12.5\% \\
& Pass@k Training & 2.57 & 31.7\% & 22.0\% & 16.8\% & 17.7\% & 11.8\% \\
\cmidrule{2-8}
& \textbf{DIVER(ours)} & \textbf{1.52} & \textbf{76.4\%} & 7.5\% & 7.1\% & 3.8\% & 5.2\% \\
\midrule
\multirow{5}{*}{AIME24} 
& Qwen2.5-Math-7B & 3.53 & 6.7\% & 23.3\% & 16.7\% & 16.7\% & 36.7\% \\
& GPRO w/ Clip-Higher & 3.13 & 6.7\% & 30.0\% & 23.3\% & 23.3\% & 16.7\% \\
& Entropy-RL & 3.47 & 6.7\% & 13.3\% & 30.0\% & 26.7\% & 23.3\% \\
& Pass@k Training & 2.73 & 26.7\% & 16.7\% & 20.0\% & 30.0\% & 6.7\% \\
\cmidrule{2-8}
& \textbf{DIVER(ours)} & \textbf{1.77} & \textbf{66.7\%} & 6.7\% & 16.7\% & 3.3\% & 6.7\% \\
\midrule
\multirow{5}{*}{AIME25} 
& Qwen2.5-Math-7B & 3.73 & 0.0\% & 30.0\% & 10.0\% & 16.7\% & 43.3\% \\
& GPRO w/ Clip-Higher & 2.83 & 16.7\% & 26.7\% & 23.3\% & 23.3\% & 10.0\% \\
& Entropy-RL & 3.00 & 20.0\% & 13.3\% & 23.3\% & 33.3\% & 10.0\% \\
& Pass@k Training & 2.70 & 26.7\% & 20.0\% & 26.7\% & 10.0\% & 16.7\% \\
\cmidrule{2-8}
& \textbf{DIVER(ours)} & \textbf{1.83} & \textbf{60.0\%} & 13.3\% & 16.7\% & 3.3\% & 6.7\% \\
\midrule
\multirow{5}{*}{AMC} 
& Qwen2.5-Math-7B & 3.17 & 20.5\% & 18.1\% & 15.7\% & 15.7\% & 30.1\% \\
& GPRO w/ Clip-Higher & 2.73 & 21.7\% & 20.5\% & 31.3\% & 15.7\% & 10.8\% \\
& Entropy-RL & 2.82 & 26.5\% & 18.1\% & 18.1\% & 21.7\% & 15.7\% \\
& Pass@k Training & 2.49 & 31.3\% & 25.3\% & 16.9\% & 15.7\% & 10.8\% \\
\cmidrule{2-8}
& \textbf{DIVER(ours)} & \textbf{1.27} & \textbf{85.5\%} & 8.4\% & 1.2\% & 3.6\% & 1.2\% \\
\midrule
\multirow{5}{*}{MATH-500} 
& Qwen2.5-Math-7B & 2.41 & 39.8\% & 21.6\% & 13.0\% & 9.0\% & 16.6\% \\
& GPRO w/ Clip-Higher & 2.27 & 43.8\% & 18.2\% & 16.2\% & 11.0\% & 10.8\% \\
& Entropy-RL & 2.24 & 43.4\% & 20.0\% & 13.6\% & 14.8\% & 8.2\% \\
& Pass@k Training & 2.19 & 48.2\% & 16.2\% & 13.4\% & 12.6\% & 9.6\% \\
\cmidrule{2-8}
& \textbf{DIVER(ours)} & \textbf{1.40} & \textbf{82.6\%} & 5.0\% & 4.8\% & 4.8\% & 2.8\% \\
\midrule
\multirow{5}{*}{Minerva} 
& Qwen2.5-Math-7B & 3.14 & 12.1\% & 26.8\% & 21.3\% & 14.3\% & 25.4\% \\
& GPRO w/ Clip-Higher & 2.93 & 17.6\% & 22.8\% & 24.3\% & 19.5\% & 15.8\% \\
& Entropy-RL & 2.79 & 18.0\% & 25.4\% & 26.1\% & 20.6\% & 9.9\% \\
& Pass@k Training & 2.70 & 25.4\% & 25.0\% & 16.2\% & 21.3\% & 12.1\% \\
\cmidrule{2-8}
& \textbf{DIVER(ours)} & \textbf{1.41} & \textbf{82.4\%} & 5.9\% & 4.4\% & 2.6\% & 4.8\% \\
\midrule
\multirow{5}{*}{Olympiad} 
& Qwen2.5-Math-7B & 3.18 & 15.7\% & 22.1\% & 18.5\% & 16.2\% & 27.6\% \\
& GPRO w/ Clip-Higher & 2.72 & 20.4\% & 25.4\% & 28.0\% & 14.0\% & 12.1\% \\
& Entropy-RL & 2.81 & 21.1\% & 22.8\% & 20.9\% & 24.2\% & 10.9\% \\
& Pass@k Training & 2.60 & 25.9\% & 28.5\% & 17.3\% & 16.4\% & 11.9\% \\
\cmidrule{2-8}
& \textbf{DIVER(ours)} & \textbf{1.45} & \textbf{81.2\%} & 5.9\% & 4.0\% & 4.8\% & 4.0\% \\
\bottomrule
\end{tabular}%
}
\end{table}

\subsection{Analysis of Model Scale and Reasoning Horizon}

We further explore DIVER's adaptability across various language models, including \textit{small, weak} or \textit{different architecture} models. As shown in Figure.\ref{fig:model_scacle}, DIVER consistently outperforms baselines across all model variants, demonstrating its strong generalization capabilities.
\begin{table}[ht]
\centering
\caption{Overall performance on six competition-level benchmark performance on Qwen2.5-Math-1.5B, Qwen2.5-7B-Base,  LLaMA-3.1-8B-Instruct \textcolor{my_blue}{and DeepSeek-R1-Distill-Qwen-7B.}}
\label{tab:main_other_models}
\resizebox{\textwidth}{!}{
\begin{tabular}{lccccc>{}c|ccc>{}c}
\toprule
\multirow{2}{*}{\textbf{Model}} & \multicolumn{6}{c}{\textbf{In-Distribution Performance}} & \multicolumn{4}{c}{\textbf{Out-of-Distribution Performance}} \\
\cmidrule(lr){2-7} \cmidrule(lr){8-11}
 & \textbf{AIME 24/25} & \textbf{AMC} & \textbf{MATH-500} & \textbf{Minerva} & \textbf{Olympiad} & \textbf{Avg.} & \textbf{ARC-c} & \textbf{GPQA}$^{*}$ & \textbf{MMLU-Pro} & \textbf{Avg.} \\
\midrule
\normalrow
\multicolumn{11}{c}{\textit{Qwen2.5-Math-1.5B}} \\
\midrule
% Qwen2.5-Math-1.5B & x/x  & x  & x  & x  & x    & x & x & x  & x & x \\
GRPO w/ Clip-higher & 9.9/7.7  & 42.8  & 68.6  & 20.6  & 33.4  & 30.5   & 59.5 & 30.1 & 30.2  & 39.9 \\
\textbf{DIVER} & 11.0/8.3  & 44.3  & 71.8  & 23.9  & 33.6 & 32.2    & 33.2 & 59.1  & 31.8 & 41.4 \\
\midrule
\normalrow
\multicolumn{11}{c}{\textit{Qwen2.5-7B-Base}} \\
\midrule
% Qwen2.5-7B-Base & 7.6/3.3  & 29.4  & 56.6  & 16.5  & 26.0    & 23.2 & x & x  & x & x \\
GRPO w/ Clip-higher & 15.8/12.1  & 49.2  & 78.2  & 27.6  & 39.1    & 37.0 & 89.8 & 28.6 & 55.4 & 57.9 \\
\textbf{DIVER} & 20.9/12.9  & 48.9  & 76.0  & 30.9  & 42.8  & 38.7  & 91.1 & 31.1 & 55.2 & 59.2 \\
\midrule
\normalrow
\multicolumn{11}{c}{\textit{LLaMA-3.1-8B-Instruct}} \\
\midrule
% LLaMA-3.1-8B-Instruct & 4.7/0.4  & 15.5  & 42.0  & 14.0  & 12.6 & 14.9    & x & x  & x & x \\
GRPO w/ Clip-higher & 6.9/1.4  & 23.3  & 53.0  & 23.5  & 22.0  & 21.8   & 89.2 & 14.8 & 50.8 & 51.6 \\
\textbf{DIVER} & 7.7/1.5  & 28.8  & 56.4  & 22.8  & 21.7    & 23.1   & 88.2 & 20.4 & 52.0 & 53.5 \\
% \midrule
% \normalrow
% \multicolumn{11}{c}{\textit{DeepSeek-R1-Distill-Qwen-7B}} \\
% \midrule
% % DeepSeek-R1-Distill-Qwen-7B & 40.6/30.0 & 72.3 & 85.8 & 28.3 & 51.1 & 51.4 & 27.4 & 0.0 & 45.8 & 24.4 \\
% GPRO w/ Clip-Higher & 49.3/35.0 & 81.7 & 90.2 & 35.7 & 59.0 & 58.5 & 85.2 & 2.6 & 54.9 & 47.5 \\
% \textbf{DIVER(ours)} & 51.1/36.9 & 82.0 & 92.6 & 36.8 & 61.2 & 60.1 & 85.6 & 10.7 & 56.6 & 50.9 \\
\midrule
\normalrow
\multicolumn{11}{c}{\textcolor{my_blue}{\textit{DeepSeek-R1-Distill-Qwen-7B}}} \\
\textcolor{my_blue}{GPRO w/ Clip-Higher} & \textcolor{my_blue}{49.3/35.0} & \textcolor{my_blue}{81.7} & \textcolor{my_blue}{90.2} & \textcolor{my_blue}{35.7} & \textcolor{my_blue}{59.0} & \textcolor{my_blue}{58.5} & \textcolor{my_blue}{85.2} & \textcolor{my_blue}{2.6} & \textcolor{my_blue}{54.9} & \textcolor{my_blue}{47.5} \\
\textcolor{my_blue}{\textbf{DIVER(ours)}} & \textcolor{my_blue}{51.1/36.9} & \textcolor{my_blue}{82.0} & \textcolor{my_blue}{92.6} & \textcolor{my_blue}{36.8} & \textcolor{my_blue}{61.2} & \textcolor{my_blue}{60.1} & \textcolor{my_blue}{85.6} & \textcolor{my_blue}{10.7} & \textcolor{my_blue}{56.6} & \textcolor{my_blue}{50.9} \\
\bottomrule
\end{tabular}}
\end{table}

\textcolor{my_blue}{To further validate DIVER's effectiveness on models with longer reasoning horizons, we conducted additional experiments using DeepSeek-R1-Distill-Qwen-7B as the base model, which typically generates much longer responses (2500-3500 tokens). The results in Table~\ref{tab:main_other_models} show that DIVER maintains its consistent superiority even with these extended reasoning processes, achieving 60.1\% average in-domain performance and 50.9\% out-of-domain performance, compared to GPRO w/ Clip-Higher's 58.5\% and 47.5\% respectively. This confirms that our approach scales effectively to longer-horizon reasoning scenarios.}

\subsection{More Comprehensive Ablations}
\begin{wrapfigure}{r}{0.31\textwidth}
  \centering 
  \vspace{-30pt}
  \includegraphics[width=\linewidth]{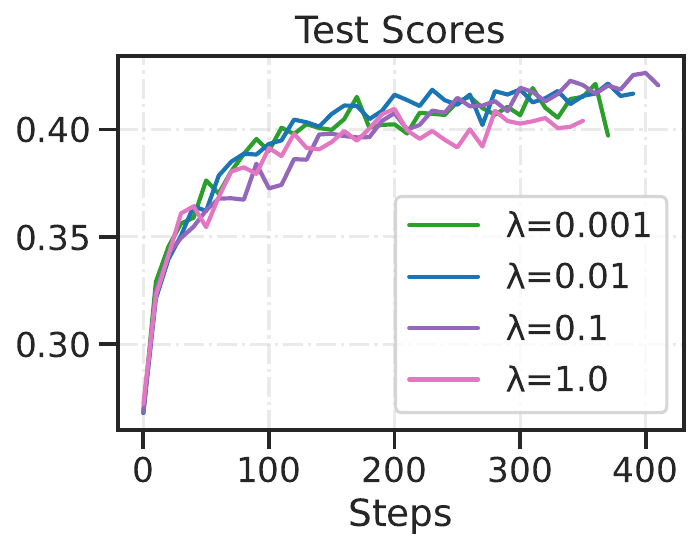}
  \vspace{-20pt}
  \caption{Average test scores with varying coefficients ($\lambda$).}
  \vspace{-10pt}
  \label{fig:hyper}
\end{wrapfigure}
\paragraph{Analysis of diversity reward shaping coefficient.} We further analyze the sensitivity of DIVER to the diversity reward shaping coefficient $\lambda$. Fig.~\ref{fig:hyper} demonstrates performance remains stable across different $\lambda$ values, with test score curves following similar trajectories regardless of coefficient magnitude. Even when $\lambda$ becomes relatively large, we observe only minor performance degradation without triggering reward hacking. This robustness to hyperparameter selection makes DIVER practical for real-world applications, as it doesn't require precise tuning of diversity reward weights.

\paragraph{Longer Horizons Improve Performance.} \textcolor{my_blue}{It is important to clarify that the "horizon" in our analysis refers to the window size used for calculating diversity metrics, rather than the full response length, in Fig.\ref{fig:tokens_appendix}. For instance, when we use a horizon of 200, we compute Textual Diversity (TD) and Equational Diversity (ED) metrics only on the first 200 tokens of a generated response, even though the complete response may be much longer (1600-1800 tokens in our main experiments).}
\begin{figure}[ht]
    \centering
    \includegraphics[width=0.9\linewidth]{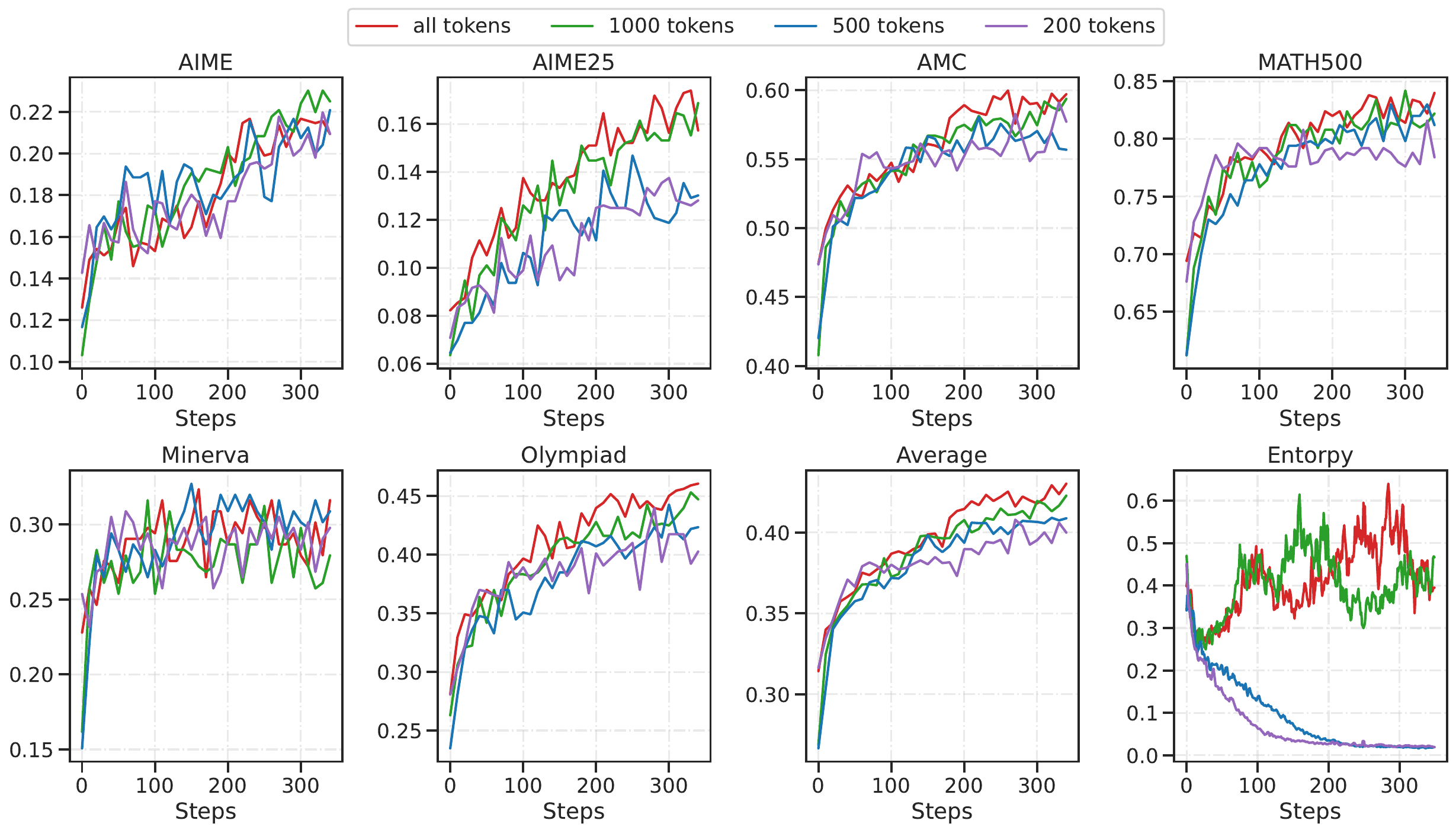}
    % \vspace{-20pt}
    \caption{DIVER performance under different token horizon constraints.}
    \vspace{-10pt}
    \label{fig:tokens_appendix}
\end{figure}

\paragraph{Hyperparameter Robustness.} \textcolor{my_blue}{To analyze hyperparameter sensitivity, we evaluate DIVER and baseline across different group sizes and training temperatures. As shown in Table~\ref{tab:hyperparameter_comparison}, DIVER demonstrates superior robustness to hyperparameter variations. When reducing group size from G = 8 to G = 4, DIVER shows a smaller relative performance decline than GPRO w/ Clip-Higher, while maintaining a 2.5 point advantage. Similarly, when decreasing temperature from t = 1.0 to t = 0.6, DIVER exhibits minimal performance fluctuation (0.3 point drop) compared to the baseline's more significant 1.2 point reduction. This stability across different configuration settings highlights DIVER's reliability in practical scenarios where optimal hyperparameter tuning may be constrained.}
\vspace{-20pt}
\begin{table}[ht]
\centering
\caption{\textcolor{my_blue}{Performance comparison across different hyperparameters.}}

\label{tab:hyperparameter_comparison}

\resizebox{0.85\textwidth}{!}{

\begin{tabular}{cccccccc c}

\toprule
\textbf{Hyperparameter} & \textbf{Model} & \textbf{AIME 24} & \textbf{AIME 25} & \textbf{AMC} & \textbf{MATH-500} & \textbf{Minerva} & \textbf{Olympiad} & \textbf{Avg.} \\
\midrule
\normalrow
\multicolumn{9}{c}{\textit{Group Size G}} \\
\midrule
\multirow{2}{*}{G = 8} & GPRO w/ Clip-Higher & 18.9 & 16.4 & 57.3 & 81.2 & 28.7 & 41.5 & 40.7 \\
& \textbf{DIVER} & 23.8 & 16.4 & 60.9 & 84.4 & 29.4 & 44.0 & 43.1 \\
\midrule
\multirow{2}{*}{G = 4} & GPRO w/ Clip-Higher & 16.3 & 12.1 & 55.8 & 78.6 & 30.9 & 39.1 & 38.8 \\
& \textbf{DIVER} & 20.5 & 14.4 & 59.6 & 81.8 & 29.4 & 42.2 & 41.3 \\
\midrule
\normalrow
\multicolumn{9}{c}{\textit{Train Temperature}} \\
\midrule
\multirow{2}{*}{t = 1.0} & GPRO w/ Clip-Higher & 18.9 & 16.4 & 57.3 & 81.2 & 28.7 & 41.5 & 40.7 \\
& \textbf{DIVER} & 23.8 & 16.4 & 60.9 & 84.4 & 29.4 & 44.0 & 43.1 \\
\midrule
\multirow{2}{*}{t = 0.6} & GPRO w/ Clip-Higher & 18.0 & 14.1 & 54.9 & 80.0 & 29.8 & 40.3 & 39.5 \\
& \textbf{DIVER} & 20.6 & 14.4 & 59.8 & 83.2 & 33.1 & 42.6 & 42.8 \\
\bottomrule
\end{tabular}}
\vspace{-400em}
\end{table}
\clearpage
\section{Case Study}\label{appendix:case}
\begin{tcolorbox}[center,breakable,title=\Large\centering{Question and Answer}]
\columnseprule=0.6pt
\textbf{Prompt:} \textit{\textless system prompt \textgreater} This is the problem: For $k \geq 3$, we define an ordered $k$-tuple of real numbers $\left(x_{1}, x_{2}, \ldots, x_{k} \right)$ to be special if, for every $i$ such that $1 \leq i \leq k$, the product $x_{1} \cdot x_{2} \cdot \ldots \cdot x_{k}=x_{i}^{2}$. Compute the smallest value of $k$ such that there are at least 2009 distinct special $k$-tuples.\\\\
\textbf{Answer:} "\$12\$"
\end{tcolorbox}

\begin{tcolorbox}[center,breakable,title=\Large\centering{DIVER Case}]
\columnseprule=0.6pt
% \begin{center}
\textcolor{my_purple}{\Large{Response 1:}}\\\\
% \end{center}
\textit{\textless think\textgreater}\\
1. If \( P = 1 \):
\[ x_1 \cdot x_2 \cdot \cdots \cdot x_k = 1. \]
Then for each \( i \):
\[ 1 = x_i^2 \implies x_i = 1 \text{ or } x_i = -1. \]
Each \( x_i \) can be either 1 or -1, and the product of all \( x_i \) must be 1. We need to count the number of such \( k \)-tuples. The number of \( k \)-tuples of \( 1 \)’s and \(-1\)’s (i.e., each \( x_i \) can be \( 1 \) or \( -1 \)) whose product is \( 1 \) can be counted as follows: the number of \( -1 \)'s must be even (since an even number of \(-1 \)'s will result in a product of \( 1\)). The total number of \( k \)-tuples where each \( x_i \) is either \( 1 \) or \(-1\) is \( 2^k \). Half of these tuples will have an even number of \(-1 \)’s (since changing the sign of any one element changes the parity of the number of \(-1 \)’s), hence there are \( 2^{k-1} \) such tuples.

2. If \( P = -1 \):
\[ x_1 \cdot x_2 \cdot \cdots \cdot x_k = -1. \]
Then for each \( i \):
\[ -1 = x_i^2 \]
This is impossible for real numbers since \( x_i^2 \geq 0 \) for all real \( x_i \). Thus, there are no solutions when \( P = -1 \) with real numbers.

Therefore, the valid solutions come only from the case \( P = 1 \), and there are \( 2^{k-1} \) such \( k \)-tuples. We need \( 2^{k-1} \geq 2009 \):
\[ 2^{k-1} \geq 2009. \]
To find the smallest \( k \):
\[ k-1 \geq \log_2 2009. \]
We calculate \( \log_2 2009 \):
\[ \log_2 2009 \approx 10.97 \]
Since \( k-1 \) must be an integer, we take the ceiling of 10.97, thus:
\[ k-1 \geq 11 \implies k \geq 12. \]
\textit{\textless /think\textgreater}\\
From the analysis above, the smallest \( k \) such that there are at least 2009 distinct special \( k \)-tuples is \( k = 12 \), because \( 2^{11} = 2048 \geq 2009 \), and \( k = 11 \) would yield only \( 2^{10} = 1024 \) tuples, which is less than 2009.

Hence, the smallest \( k \) for which there are at least 2009 distinct special \( k \)-tuples is \( k = 12 \). The final answer is:
\[
\boxed{12}
\]
\noindent\hdashrule{\linewidth}{0.5pt}{2pt 2pt}\\

\textcolor{my_purple}{\Large{Response 2:}}\\\\
\textit{\textless think\textgreater}\\
For \( P \neq 0 \), we can divide both sides by \( P \):
\[ 1 = P^{\frac{k}{2} - 1} \cdot (-1)^{k-n}. \]
For this equation to hold, two conditions must be satisfied:
1. \( P^{\frac{k}{2} - 1} = 1 \), implying \( \frac{k}{2} - 1 = 0 \) or \( P = 1 \) (because \( P \) is a real number).
2. \( (-1)^{k-n} = 1 \), implying that \( k - n \) must be even.
From \( \frac{k}{2} - 1 = 0 \), we get:
\[ \frac{k}{2} = 1 \implies k = 2. \]
However, the problem states \( k \geq 3 \), so \( k = 2 \) is not valid. Therefore, \( P = 1 \). So, \( P = 1 \) and thus \( (\sqrt{P})^2 = 1 \), meaning \( x_i^2 = 1 \) for all \( i \). Therefore, each \( x_i \) can only be \( 1 \) or \( -1 \). Now, we need \( (-1)^{k-n} = 1 \), meaning \( k - n \) is even, implying \( n \) and \( k \) have the same parity (both are either even or both are odd).
Next, let’s determine the number of distinct \( k \)-tuples. If each \( x_i \) can be either \( 1 \) or \( -1 \), there are \( 2^k \) possible \( k \)-tuples in total. However, not all of them will satisfy the condition, because they must maintain \( P = 1 \). The product \( x_1 x_2 \cdots x_k = 1 \) if and only if an even number of \( x_i \)'s are \( -1 \). The number of ways to choose an even number of \( -1 \)'s out of \( k \) positions is given by the sum of binomial coefficients of even indices:
\[ \sum_{\substack{0 \leq j \leq k \\ j \text{ even}}} \binom{k}{j}. \]
It is known that the sum of binomial coefficients of even indices equals \( 2^{k-1} \):
\[ \sum_{\substack{0 \leq j \leq k \\ j \text{ even}}} \binom{k}{j} = 2^{k-1}. \]
Thus, there are \( 2^{k-1} \) distinct special \( k \)-tuples. We need \( 2^{k-1} \geq 2009 \):
\[ 2^{k-1} \geq 2009. \]
To find the smallest \( k \):
\[ k-1 \geq \log_2(2009). \]
Since \( \log_2(2009) \approx 10.97 \), we round up to the next integer because \( k-1 \) must be an integer:
\[ k-1 \geq 11. \]
This gives \( k \geq 12 \).
\textit{\textless /think\textgreater}\\
The smallest \( k \) that satisfies \( 2^{k-1} \geq 2009 \) is \( k = 12 \). To verify, when \( k = 12 \):
\[ 2^{12-1} = 2^{11} = 2048, \]
which is indeed greater than or equal to 2009. Thus, the smallest \( k \) such that there are at least 2009 distinct special \( k \)-tuples is \( k = 12 \).
Therefore, the smallest \( k \) is:
\[
\boxed{12}
\]
\end{tcolorbox}

\begin{tcolorbox}[center,breakable,title=\Large\centering{Entropy-RL Case}]
\columnseprule=0.6pt
% \begin{center}
\textcolor{my_purple}{\Large{Response 1:}}\\\\
\textit{\textless think\textgreater}\\
the given problem states that $k\ge 3$, so we need to reconsider another approach since $k=2$ isn't applicable directly for $k\ge 3$. So let's say $P = t^2$ where $t = \sqrt{P}$ and then we need $t^2 = t^k$. If we let $t = \sqrt{P}$, we get $$t^2 = t^k \implies t^2 - t^k = 0 \implies t^2(1 - t^{k-2}) = 0.$$ This implies that $t^2 = 0$ or $t^{k-2} = 1$. First, if $t^2 = 0$, then $t = 0$, and thus all the coordinates should be zero yielding once again the trivial solution, leading to the previous case.

For $t^{k-2} = 1$, we have two solutions:

1. $t=1$ and
2. $t=-1$.

When $t = 1$, then $P = 1$, and each $x_i = \pm 1$, and when $t=-1$, $P = 1$, and again each $x_i = \pm 1$.

Hence, each coordinate can be either $1$ or $-1$, giving $2^k$ total combinations of $k$-tuples. Since we already have the trivial all-zero solution giving just one, we'll focus on the remaining ones, so the total number of solutions is $2^k - 1$. We need $2^k-1 \ge 2009$. Let's calculate smallest $k$:

We need to find the smallest integer $k$ such that

\[ 2^k - 1 \ge 2009,\] or

\[ 2^k \ge 2010.\] Now let's calculate powers of 2:

\[
\begin{aligned}
& 2^{10} = 1024, \\
& 2^{11} = 2048.
\end{aligned}
\]

It follows that the smallest value for which $2^k \ge 2010$ is when $k =11.$

\textit{\textless /think\textgreater}\\
Now that we've reasoned through the problem, we conclude:

That would mean $k = 11$ would yield $2^{11} = 2048$, and then there'd be $2048 -1 = 2047$ nonzero solutions plus 1 trivial solution leading to $2048$ total tuples; thus $k = 11$ suffices because even $k=10$ would not yield enough (as $2^{10}-1 = 1023 < 2009$). So the final answer must be:

\[
\boxed{11}
\]
\noindent\hdashrule{\linewidth}{0.5pt}{2pt 2pt}\\

\textcolor{my_purple}{\Large{Response 2:}}\\\\
\textit{\textless think\textgreater}\\
least 2009 distinct special \(k\)-tuples, we must have

\[ 2^k \geq 2009.\]

We need to find the smallest integer \(k\) such that \(2^k \geq 2009\).

Let’s calculate the powers of 2:

\[
\begin{aligned}
& 2^{10} = 1024, \\
& 2^{11} = 2048.
\end{aligned}
\]

We can see that \(2^{10} = 1024 < 2009\) and \(2^{11} = 2048 > 2009\). Thus, the smallest \(k\) such that \(2^k \geq 2009\) is 11.\\
\textit{\textless /think\textgreater}\\
The smallest value of \( k \) such that there are at least 2009 distinct special \( k \)-tuples is
\[\boxed{11}\]

\end{tcolorbox}

\begin{tcolorbox}[center,breakable,title=\Large\centering{GRPO w/ Clip-higher Case}]
\columnseprule=0.6pt
% \begin{center}
\textcolor{my_purple}{\Large{Response 1:}}\\\\
\textit{\textless think\textgreater}\\
If $x = 0$, then $P = 0$, and all $x_i$ must be zero, which gives us one more solution.

So, for $k > 3$, we have $k-1$ distinct solutions (the $k-2$ roots of unity plus the solution $x = 0$). We need at least 2009 distinct special $k$-tuples. Therefore, we need $k-1 \geq 2009$, which gives us $k \geq 2010$. The smallest value of $k$ that satisfies this inequality is $k = 2010$.

Let's verify this. For $k = 2010$, we have $2010-1 = 2009$ distinct solutions from the roots of unity plus the solution $x = 0$, which gives us exactly 2009 distinct special $k$-tuples. Therefore, the smallest value of $k$ such that there are at least 2009 distinct special $k$-tuples is $k = 2010$.
\textit{\textless /think\textgreater}\\

The smallest value of $k$ such that there are at least 2009 distinct special $k$-tuples is \[\boxed{2010}\]

\noindent\hdashrule{\linewidth}{0.5pt}{2pt 2pt}\\

\textcolor{my_purple}{\Large{Response 2:}}\\\\
\textit{\textless think\textgreater}\\
To count the number of distinct special $k$-tuples, we need to consider the number of solutions to these equations. If $k-2$ is even, we have at least 2 distinct tuples $(1, 1, \ldots, 1)$ and $(-1, -1, \ldots, -1)$. If $k-2$ is odd, we can have more complex tuples involving $-1$. 

To have at least 2009 distinct special $k$-tuples, we need to find the smallest $k$ such that the number of solutions is at least 2009. If $k-2$ is even, we have at least 2 solutions. If $k-2$ is odd, we can have more solutions, but we need to count them carefully. For $k-2$ odd, we can have tuples with an even number of $-1$'s and the rest $1$'s, and permutations of $(x, -x, -1, -1, \ldots, -1)$. 

The number of distinct tuples grows quickly with $k$, and we need to find the smallest $k$ such that the number of solutions is at least 2009. Let's calculate this step by step. For $k-2$ odd, the number of distinct tuples is related to the number of ways to choose positions for $-1$ and the values of $x$ and $-x$. 

Let's check the number of solutions for different $k$:

- For $k = 3$, we have $(1, 1, 1)$ and $(-1, -1, -1)$, so 2 solutions.
- For $k = 4$, we have $(1, 1, 1, 1)$, $(-1, -1, -1, -1)$, and permutations of $(x, -x, -1, -1)$, so more than 2 solutions.
- For $k = 5$, we have $(1, 1, 1, 1, 1)$, $(-1, -1, -1, -1, -1)$, and more solutions involving $-1$.

We can see that as $k$ increases, the number of solutions increases. We need to find the smallest $k$ such that the number of solutions is at least 2009. 

Let's calculate the number of solutions for $k = 2m+3$ (since $k-2$ must be odd):

- For $k = 2m+3$, we have at least the tuples with all $1$'s and all $-1$'s, plus permutations of $(x, -x, -1, -1, \ldots, -1)$. The number of such tuples is related to the number of ways to choose positions for $-1$ and the values of $x$ and $-x$. 

The number of distinct tuples for $k = 2m+3$ is at least $2 + \text{number of ways to choose positions for } -1 \text{ and } x, -x$. This grows quickly with $k$. 

By trial and error, we can see that for $k = 2011$, we have $k-2 = 2009$, which is odd, and we can have at least 2009 distinct tuples. 

\textit{\textless /think\textgreater}\\

The smallest value of $k$ such that there are at least 2009 distinct special $k$-tuples is $k = 2011$. 

The final answer is:
\[
\boxed{2011}
\]

\end{tcolorbox}

\begin{tcolorbox}[center,breakable,title=\Large\centering{Pass@k Training Case}]
\columnseprule=0.6pt
% \begin{center}
\textcolor{my_purple}{\Large{Response 1:}}\\\\
\textit{\textless think\textgreater}\\
Since $P = P^{k/2}$, we have two cases to consider:
1. If $P = 0$: In this case, all $x_i$ must be zero because otherwise, if any $x_i$ is non-zero, the product would not be zero. This gives us the $k$-tuple $(0, 0, \ldots, 0)$, which is indeed a solution. For $P = 0$, it forms one special $k$-tuple.
2. If $P \neq 0$: We can divide both sides of the equation by $P^{k/2}$ (since $P > 0$):
\[ 1 = (-1)^{k-a}. \]
This implies that $(-1)^{k-a} = 1$, meaning $k-a$ must be even. Consequently, $a$ and $k$ must have the same parity (both even or both odd). Moreover, since \(P = (\sqrt{P})^2 = (\sqrt{P})^k \cdot (-1)^{k-a}\), we get:
\[ 1 = (-1)^{k-a}. \]
This implies $k - a$ is even, so $a$ and $k$ are of the same parity (both even or both odd).

Next, we need to count the number of distinct $k$-tuples. For non-zero $P$, the number of distinct $k$-tuples is the number of ways to choose $a$ (with $0 \leq a \leq k$) such that $a$ and $k$ share the same parity. The number of choices for $a$ (with $0 \leq a \leq k$) having the same parity as $k$ is $\frac{k+1}{2}$ if $k$ is odd and $\frac{k}{2} + 1$ if $k$ is even. Including the zero product case, we need the total number of distinct $k$-tuples to be at least 2009:
\[ 1 + \text{(number of $k$-tuples with $P \neq 0$)}. \]
If $k$ is odd, there are $\frac{k+1}{2}$ choices for $a$, making the total number of distinct $k$-tuples:
\[ 1 + \frac{k+1}{2}. \]
If $k$ is even, there are $\frac{k}{2} + 1$ choices for $a$, making the total number of distinct $k$-tuples:
\[ 1 + \left(\frac{k}{2} + 1\right) = 1 + \frac{k}{2} + 1 = \frac{k}{2} + 2. \]
We need this total to be at least 2009. We will check both cases separately.
1. For $k$ odd: \( 1 + \frac{k+1}{2} \geq 2009 \)
\[ \frac{k+1}{2} + 1 \geq 2009, \]
\[ \frac{k+1}{2} \geq 2008, \]
\[ k+1 \geq 4016, \]
\[ k \geq 4015. \]
2. For $k$ even: \( \frac{k}{2} + 2 \geq 2009 \)
\[ \frac{k}{2} + 2 \geq 2009, \]
\[ \frac{k}{2} \geq 2007, \]
\[ k \geq 4014. \]
The smallest $k$ that satisfies either condition is \(k = 4015\) (since \(4015\) is odd and satisfies the first condition). Hence, the smallest $k$ is 4015.\\
\textit{\textless /think\textgreater}\\
To find the smallest $k$ such that there are at least 2009 distinct special $k$-tuples, we need to satisfy the condition based on the derived inequalities. We found that for \( k \geq 4015 \) (when $k$ is odd) or \( k \geq 4014 \) (when $k$ is even):

1. If $k$ is odd, the smallest $k$ satisfying $k \geq 4015$.
2. If $k$ is even, the smallest $k$ satisfying $k \geq 4014$.

Since \(4015\) is the smallest value from these two cases, it gives us the final answer:
\[\boxed{4015}\]

\noindent\hdashrule{\linewidth}{0.5pt}{2pt 2pt}\\

\textcolor{my_purple}{\Large{Response 2:}}\\\\
\textit{\textless think\textgreater}\\
Now, let's consider the possible values of $x_i$. Since $P = x_i^2$, and $P$ is either $1$ or $-1$ (if $k-2$ is even), we have $x_i^2 = 1$ or $x_i^2 = -1$. The equation $x_i^2 = 1$ has two solutions: $x_i = 1$ or $x_i = -1$. The equation $x_i^2 = -1$ has no real solutions. Therefore, each $x_i$ must be either $1$ or $-1$.

If $P = 1$, then the number of $-1$'s in the tuple must be even. If $P = -1$, then the number of $-1$'s in the tuple must be odd. We need to count the number of such tuples.

Let's denote the number of $-1$'s in the tuple by $m$. If $P = 1$, then $m$ must be even, and if $P = -1$, then $m$ must be odd. The number of ways to choose $m$ positions out of $k$ for the $-1$'s is $\binom{k}{m}$. So, the total number of special $k$-tuples is the sum of $\binom{k}{m}$ for all even $m$ plus the sum of $\binom{k}{m}$ for all odd $m$. It is known that the sum of the binomial coefficients for even $m$ and odd $m$ are equal and each sum to $2^{k-1}$. Therefore, the total number of special $k$-tuples is $2^{k-1} + 2^{k-1} = 2^k$.

We need this to be at least 2009. So, we need to find the smallest $k$ such that $2^k \geq 2009$. We can calculate the powers of 2:
\[ 2^{10} = 1024, \]
\[ 2^{11} = 2048. \]
Since $2048 \geq 2009$, the smallest $k$ is 11. We can verify that $2^{11} = 2048$, which is indeed greater than or equal to 2009.

Therefore, the smallest value of $k$ such that there are at least 2009 distinct special $k$-tuples is $11$.

\textit{\textless /think\textgreater}\\

The smallest value of $k$ such that there are at least 2009 distinct special $k$-tuples is \[\boxed{11}\]

\end{tcolorbox}

\clearpage
\section{The Use of Large Language Models (LLMs)}
We utilize Large Language Models (LLMs) to assist with polishing the writing and improving text readability. Specifically, LLMs are employed for proofreading, enhancing grammar, and refining sentence structure. The LLM was used solely for editorial purposes to improve clarity and did not contribute to research ideation, experimental design, implementation, analysis, or scientific conclusions. All core research contributions, experiments, and analyses were conducted independently by the authors without LLM assistance.

\end{document}